\documentclass{article}

\usepackage[final]{nips_2016}

\usepackage[utf8]{inputenc} %
\usepackage[T1]{fontenc}    %
\usepackage{hyperref}       %
\usepackage{url}            %
\usepackage{booktabs}       %
\usepackage{amsfonts}       %
\usepackage{nicefrac}       %
\usepackage{microtype}      %
\usepackage{listings}
\usepackage{tikz}
\usepackage{subfigure}
\usepackage[english]{babel}
\usepackage{graphicx,caption,centernot,amssymb,amsthm,algorithm,algorithmicx,algpseudocode,mathtools}
\usepackage{amsmath}
\usepackage{bm}
\usepackage{paralist}

\DeclareMathOperator*{\argmin}{arg\,min}

\definecolor{blue}{rgb}{0,0,1}
\definecolor{orange}{rgb}{1,0.5,0.1}
\definecolor{brightblue}{rgb}{0.92,0.92,1}
\definecolor{brightgrey}{rgb}{0.96,0.96,1}
\definecolor{green}{rgb}{0.2,1.0,0.2}

\lstdefinestyle{ASP}{
  belowcaptionskip=1\baselineskip,
  breaklines=true,
  frame=L,
  xleftmargin=\parindent,
  language=psl,
  showstringspaces=false,
  basicstyle=\script\ttfamily,
  keywordstyle=\bfseries\color{green!40!black},
  commentstyle=\itshape\color{purple!40!black},
  identifierstyle=\color{blue},
  stringstyle=\color{orange},
}

\newcommand{\CI}{\mathrel{\perp\mspace{-10mu}\perp}}
\newcommand{\nCI}{\centernot{\CI}}
\newtheorem{lemma}{Lemma}
\newcommand{\name}{ACI}
\newcommand\idcausal{\dashrightarrow}
\newcommand\idacausal{\text{${}\not\!\dashrightarrow{}$}}
\newcommand\B[1]{\bm{#1}}
\newcommand\given{\,|\,}

\newcommand\causes{\idcausal}
\newcommand\notcauses{\idacausal}
\newcommand{\xto}[1]{\stackrel{#1}{\to}}
\newcommand\eref[1]{(\ref{#1})}
\newtheorem{theorem}{Theorem}

\newcommand\Prb{\mathbb{P}}
\newcommand{\ck}{\checkmark}

\definecolor{lgray}{rgb}{0.9,0.9,0.9}

\newcommand\loss{\mathcal{L}}
\DeclareMathOperator*{\DAG}{DAG}

\title{Ancestral Causal Inference}

\author{
  Sara Magliacane\\
  VU Amsterdam \& University of Amsterdam\\
  \texttt{sara.magliacane@gmail.com}
  \And
  Tom Claassen\\
  Radboud University Nijmegen\\
  \texttt{tomc@cs.ru.nl} 
    \And
  Joris M. Mooij\\
  University of Amsterdam\\
  \texttt{j.m.mooij@uva.nl}   
}

\begin{document}

\maketitle

\begin{abstract}%
Constraint-based causal discovery from limited data is a notoriously difficult challenge due to the many borderline independence test decisions. 
Several approaches to improve the reliability of the predictions
by exploiting redundancy in the independence information have been proposed recently. 
Though promising, existing approaches can still be greatly improved in terms of accuracy and scalability.
We present a novel method that reduces the combinatorial
  explosion of the search space by using a more coarse-grained representation of causal information, drastically reducing computation time.
Additionally, we propose a method to score causal predictions based on their confidence.
Crucially, our implementation also allows one to easily combine observational and interventional data and to incorporate various types of available background knowledge. 
We prove soundness and asymptotic consistency of our method and demonstrate that it can outperform the state-of-the-art on synthetic data, achieving a speedup of several orders of magnitude. We illustrate its practical feasibility by applying it to a challenging protein data set.
\end{abstract}

\section{Introduction}
Discovering causal relations from data is at the foundation of the scientific method. Traditionally, cause-effect relations have been recovered from experimental data in which the variable of interest is perturbed, but seminal work like the \textit{do}-calculus \cite{Pearl2009} and the PC/FCI algorithms \cite{Spirtes2000,Zhang:2008:COR:1414091.1414237} demonstrate that, under certain assumptions (e.g., the well-known \emph{Causal Markov} and \emph{Faithfulness} assumptions \cite{Spirtes2000}), it is already possible to obtain substantial causal information by using only observational data. 

Recently, there have been several proposals for combining observational and experimental data to discover causal relations. These causal discovery methods are usually divided into two categories: constraint-based and score-based methods. Score-based methods typically evaluate models using a penalized likelihood score, while constraint-based methods use statistical independences to express constraints over possible causal models. The advantages of constraint-based over score-based methods are the ability to handle latent confounders and selection bias naturally, and that there is no need for parametric modeling assumptions. Additionally, constraint-based methods expressed in \emph{logic} \cite{loci,BCCD,triantafillou2015constraint,antti} allow for an easy integration of background knowledge, which is not trivial even for simple cases in approaches that are not based on logic \cite{borboudakis2012}. 

Two major disadvantages of traditional constraint-based methods are: (i) vulnerability to errors in statistical independence test results, which are quite common in real-world applications, (ii) no ranking or estimation of the confidence in the causal predictions.
Several approaches address the first issue and improve the reliability of constraint-based methods by exploiting redundancy in the independence information \cite{BCCD,antti,triantafillou2015constraint}. 
The idea is to assign weights to the input statements that reflect their reliability, and then use a reasoning scheme that takes these weights into account. 
Several weighting schemes can be defined, from simple ways to attach weights to single independence statements \cite{antti}, to more complicated schemes to obtain weights for combinations of independence statements \cite{triantafillou2015constraint,BCCD}.
Unfortunately, these approaches have to sacrifice either accuracy by using a greedy method \cite{BCCD,triantafillou2015constraint}, or scalability by formulating a discrete optimization problem on a super-exponentially large search space \cite{antti}. 
Additionally, the confidence estimation issue is addressed only in limited cases \cite{ICP}.

We propose Ancestral Causal Inference (ACI), a logic-based method that provides comparable accuracy to the best state-of-the-art constraint-based methods (e.g., \cite{antti}) for causal systems with latent variables without feedback, but improves on their scalability by using a more coarse-grained representation of causal information.  Instead of representing all possible direct causal relations, in ACI we represent and reason only with ancestral relations (``indirect'' causal relations), developing specialised ancestral reasoning rules. This representation, though still super-exponentially large, drastically reduces computation time. Moreover, it turns out to be very convenient, because in real-world applications the distinction between direct causal relations and ancestral relations is not always clear or necessary. Given the estimated ancestral relations, the estimation can be refined to direct causal relations by constraining standard methods to a smaller search space, if necessary.

Furthermore, we propose a method to score predictions according to their confidence. The confidence score can be thought of as an approximation to the marginal probability of an ancestral relation. Scoring predictions enables one to rank them according to their reliability, allowing for higher accuracy. This
is very important for practical applications, as the low reliability of the predictions of constraint-based methods has been a major impediment to their wide-spread use.%

We prove soundness and asymptotic consistency under mild conditions on the statistical tests for ACI and our scoring method. 
We show that ACI outperforms standard methods, like bootstrapped FCI and CFCI, in terms of accuracy, and achieves a speedup of several orders of magnitude over \cite{antti} on a synthetic dataset.
We illustrate its practical feasibility by applying it to a challenging protein data set \cite{SPP05} that so far had only been addressed with score-based methods and observe that it successfully recovers from faithfulness violations. 
In this context, we showcase the flexibility of logic-based approaches by introducing weighted ancestral relation constraints that we obtain from a combination of observational and interventional data, and show that they substantially increase the reliability of the predictions. 
Finally, we provide an open-source version of our algorithms and the evaluation framework, which can be easily extended, at \url{http://github.com/caus-am/aci}.

\section{Preliminaries and related work} \label{bckg}

\paragraph{Preliminaries}
We assume that the data generating process can be modeled by a causal Directed Acyclic Graph (DAG) that may contain latent variables. For simplicity we also assume that there is no selection bias. Finally, we assume that the \emph{Causal Markov Assumption} and the \emph{Causal Faithfulness
Assumption} \cite{Spirtes2000} both hold.  In other words, the conditional
independences in the observational distribution correspond one-to-one
with the d-separations in the causal DAG. 
Throughout the paper we represent variables with uppercase letters, while sets of variables are denoted by boldface. All proofs are provided in the Supplementary Material.

A directed edge $X\to Y$ in the causal DAG represents a \emph{direct causal relation} between cause $X$ on effect $Y$. 
Intuitively, in this framework this indicates that manipulating $X$ will produce a change in $Y$, while manipulating $Y$ will have no effect on $X$. A more detailed discussion can be found in \cite{Spirtes2000}.
A sequence of directed edges $X_1 \to X_2 \to \dots \to X_n$ 
is a \emph{directed path}. If there exists a directed path from $X$ to $Y$ (or $X=Y$), then $X$ is an \emph{ancestor} of $Y$ (denoted as $X \causes Y$). 
Otherwise, $X$ is not an ancestor of $Y$
(denoted as $X \notcauses Y$). 
For a set of variables $\B{W}$, we write:
\begin{equation}
\begin{split}
X \causes \B{W} := \exists Y \in \B{W} : X \causes Y, \\
X \notcauses \B{W} \ := \forall Y \in \B{W} : X \notcauses Y.
\end{split}
\end{equation}
We define an \emph{ancestral structure} as any non-strict partial order on the observed variables of the DAG, i.e., any relation that satisfies the following axioms:
\begin{align}
  & \text{(\emph{reflexivity})}: X \causes X, \label{eq:reflexivity} \\ 
  & \text{(\emph{transitivity})}: X \causes Y \ \wedge \ Y \causes Z \implies  X \causes Z, \label{eq:transitivity} \\
  & \text{(\emph{antisymmetry})}: X \causes Y \land Y \causes X \implies X = Y. \label{eq:acyclicity}
\end{align} 
The underlying causal DAG induces a \emph{unique} ``true'' ancestral structure, which represents the transitive closure of the direct causal relations projected on the observed variables.

For disjoint sets $\B{X},\B{Y},\B{W}$ we denote conditional independence of $\B{X}$ and $\B{Y}$ given $\B{W}$ as $\B{X} \CI \B{Y} \,|\, \B{W}$, and conditional dependence as $\B{X} \nCI \B{Y} \,|\, \B{W}$. We call the cardinality $|\B{W}|$ the \emph{order} of the conditional (in)dependence relation.
Following \cite{loci} we define a \emph{minimal conditional independence} by:
$$X \CI Y \given \B{W} \cup [Z] :=  (X \CI Y \given \B{W} \cup Z) \wedge (X \nCI Y \given \B{W}),$$
and similarly, a \emph{minimal conditional dependence} by:
$$X \nCI Y \given \B{W} \cup [Z] := (X \nCI Y \given \B{W} \cup Z) \wedge (X \CI Y \given \B{W}).$$ 
The square brackets indicate that $Z$ is needed for the
(in)dependence to hold in the context of $\B{W}$.
Note that the negation of a minimal conditional independence is not a minimal conditional dependence. 
Minimal conditional (in)dependences are closely related to ancestral relations,
as pointed out in \cite{loci}:
\begin{lemma}\label{lemm:Claassen} For disjoint (sets of) variables $X, Y, Z, \B{W}$:
  \begin{align}
    X \CI Y \given \B{W} \cup [Z] & \implies Z \causes (\{X,Y\} \cup \B{W}), \label{eq:causes_rule} \\
    X \nCI Y \given \B{W} \cup [Z] & \implies Z \notcauses (\{X,Y\} \cup \B{W}). \label{eq:notcauses_rule} 
  \end{align}
\end{lemma}
Exploiting these rules (as well as others that will be introduced in Section~\ref{aci}) to deduce ancestral relations directly from (in)dependences is key to the greatly improved scalability of our method.

\paragraph{Related work on conflict resolution}

One of the earliest algorithms to deal with conflicting inputs in constraint-based causal discovery is Conservative PC
\cite{conf/uai/RamseyZS06}, which adds ``redundant'' checks to the PC algorithm that allow it to detect 
inconsistencies in the inputs, and then makes only predictions that do not rely on the ambiguous inputs. 
The same idea can be applied to FCI, yielding Conservative FCI (CFCI) \cite{Colombo++2012,pcalg}. 
BCCD (Bayesian Constraint-based Causal Discovery) \cite{BCCD} uses Bayesian confidence estimates to process information
in decreasing order of reliability, discarding contradictory inputs as they
arise.
COmbINE (Causal discovery from Overlapping INtErventions)  
\cite{triantafillou2015constraint} is an algorithm that combines 
the output of FCI on several overlapping observational and experimental datasets into a single causal model by first pooling and recalibrating the independence test $p$-values, and then adding each constraint incrementally in order of reliability to a SAT instance. Any constraint that makes the problem unsatisfiable is discarded.

Our approach is inspired by a method presented by Hyttinen, Eberhardt and J{\"{a}}rvisalo \cite{antti} (that we will refer to as HEJ in this paper), in which causal discovery is formulated as a constrained discrete minimization problem. 
Given a list of weighted independence statements, HEJ searches for the optimal causal graph $\mathcal{G}$ (an acyclic directed mixed graph, or ADMG) that minimizes the sum of the weights of the independence statements that are violated according to $\mathcal{G}$. In order to test whether a causal graph $\mathcal{G}$ induces a certain independence, the method creates an \emph{encoding DAG of d-connection graphs}.
D-connection graphs are graphs that can be obtained from a causal graph through a series of operations (conditioning, marginalization and interventions). An encoding DAG of d-connection graphs is a complex structure encoding all possible d-connection graphs and the sequence of operations that generated them from a given causal graph. This approach has been shown to correct errors in the inputs, but is computationally demanding because of the huge search space.

\section{\name: Ancestral Causal Inference}\label{aci}

We propose Ancestral Causal Inference (\name), a causal discovery method that accurately reconstructs ancestral structures, also in the presence of latent variables and statistical errors.
\name\ builds on HEJ \cite{antti}, but rather than optimizing over encoding DAGs, \name\ optimizes over the much simpler (but still very expressive) ancestral structures.

For $n$ variables, the number of possible ancestral structures is the number of partial orders (\url{http://oeis.org/A001035}), which grows as $2^{n^2/4 + o(n^2)}$ \cite{asymptoticEnumerationPosets}, while the number of DAGs can be computed with a well-known super-exponential recurrence formula (\url{http://oeis.org/A003024}). The number of ADMGs is $|\DAG(n)| \times 2^{n(n-1)/2}$. Although still super-exponential, the number of ancestral structures grows asymptotically much slower than the number of DAGs and even more so, ADMGs.
For example, for 7 variables, there are $6\times 10^6$ ancestral structures but already $2.3\times 10^{15}$ ADMGs, which lower bound the number of encoding DAGs of d-connection graphs used by HEJ.

\paragraph{New rules} The rules in HEJ explicitly encode marginalization and conditioning operations on d-connection graphs, so they cannot be easily adapted to work directly with ancestral relations. 
Instead, \name\ encodes the ancestral reasoning rules \eref{eq:reflexivity}--\eref{eq:notcauses_rule} and five novel causal reasoning rules:
\begin{lemma}\label{new_rules}
For disjoint (sets) of variables $X,Y,U,Z,\B{W}$:
  \begin{align}
    & (X \CI Y \mid \B{Z}) \wedge (X \notcauses \B{Z}) \implies X \notcauses Y, \label{eq:entner_rule} \\
    & X \nCI Y \mid \B{W} \cup [Z] \implies X \nCI Z \mid \B{W}, \\
    & X \CI Y \mid \B{W} \cup [Z] \implies X \nCI Z \mid \B{W}, \\
    & (X \CI Y \mid \B{W} \cup [Z]) \land (X \CI Z \mid \B{W} \cup U) \implies (X \CI Y \mid \B{W} \cup U), \\
    & (Z \nCI X \mid \B{W}) \land (Z \nCI Y \mid \B{W}) \land (X \CI Y \mid \B{W}) \implies X \nCI Y \mid \B{W} \cup Z. \label{eq:lastrule}
  \end{align}
\end{lemma}
We prove the soundness of the rules in the Supplementary Material. We elaborate some conjectures about their completeness in the discussion after Theorem \ref{eq:soundness} in the next Section.

\paragraph{Optimization of loss function}
We formulate causal discovery as an optimization
problem where a loss function is optimized over possible causal structures. Intuitively, the loss function sums the weights of all the inputs that are violated in a candidate
causal structure. 

Given a list $I$ of weighted input statements $(i_j, w_j)$, where $i_j$ is the input statement and $w_j$ is the associated weight, we define the loss function as the sum of the weights of the input statements that are not satisfied in a given possible structure $W \in \mathcal{W}$, where $\mathcal{W}$ denotes the set of all possible causal structures.
Causal discovery is formulated as a discrete optimization problem:
\begin{equation} \label{minimization}
  W^* = \argmin_{W \in \mathcal{W}} \loss(W;I),
\end{equation}
\begin{equation}\label{loss}
  \loss(W;I) := \sum_{(i_j, w_j) \in I: \ W \cup \mathcal{R} \models \lnot i_j} w_j,
\end{equation}
where $W \cup \mathcal{R} \models \lnot i_j$ means that input $i_j$ is not satisfied in structure $W$ according to the rules $\mathcal{R}$.

This general formulation includes both HEJ and ACI, which differ in the types of possible structures $\mathcal{W}$ and the rules $\mathcal{R}$.
In HEJ $\mathcal{W}$ represents all possible causal graphs (specifically, acyclic directed mixed graphs, or ADMGs, in the acyclic case) and $\mathcal{R}$ are operations on d-connection graphs. In \name\ $\mathcal{W}$ represent ancestral structures (defined with the rules\eref{eq:reflexivity}-\eref{eq:acyclicity}) and the rules $\mathcal{R}$ are rules \eref{eq:causes_rule}--\eref{eq:lastrule}.

\paragraph{Constrained optimization in ASP}
The constrained optimization problem in \eref{minimization} can be implemented using a variety of methods. Given the complexity of the rules, a formulation in an expressive logical language that supports optimization, e.g., Answer Set Programming (ASP), is very convenient.
ASP is a widely used declarative programming language based on the stable model semantics \cite{DBLP:conf/aaai/Lifschitz08,DBLP:reference/fai/Gelfond08} that has successfully been applied to several NP-hard problems.
For ACI we use the state-of-the-art ASP solver \texttt{clingo 4} \cite{clingo}. We provide the encoding in the Supplementary Material. 

\paragraph{Weighting schemes}
\name\ supports two types of input statements: conditional independences and ancestral relations. These statements can each be assigned a weight that reflects their confidence. We propose two simple approaches with the desirable properties of making \name\ asymptotically consistent under mild assumptions (as described in the end of this Section), and assigning a much smaller weight to independences than to dependences (which agrees with the intuition that one is confident about a measured strong dependence, but not about independence vs.\ weak dependence). 
The approaches are:
\begin{compactitem}
\item a \emph{frequentist} approach, in which for any appropriate frequentist statistical test with independence as null hypothesis (resp.\ a non-ancestral relation), we define the weight: 
\begin{equation}\label{eq:weightFreq}
w = |\log p - \log \alpha|, \text{where} \ p = \text{$p$-value of the test}, \alpha = \text{significance level (e.g., 5\%)};
\end{equation}
\item a \emph{Bayesian} approach, in which the weight of each input statement $i$ using data set $\mathcal{D}$ is:
\begin{equation}\label{eq:weightBayes}
  w = \log \frac{p(i | \mathcal{D})}{p(\lnot i| \mathcal{D})}
  = \log \frac{p(\mathcal{D} | i)}{p(\mathcal{D} | \lnot i)} \frac{p(i)}{p(\lnot i)},
\end{equation}
where the prior probability $p(i)$ can be used as a tuning parameter.
\end{compactitem}
Given observational and interventional data, in which each intervention has a single known target (in particular, it is not a \emph{fat-hand} intervention \cite{EatonMurphy07}), a simple way to obtain a 
weighted ancestral statement $X \causes Y$ is with a two-sample test that tests whether
the distribution of $Y$ changes with respect to its observational distribution when intervening on $X$.
This approach conveniently applies to various types of interventions: perfect interventions \cite{Pearl2009}, soft interventions \cite{Markowetz++2005}, mechanism changes \cite{TianPearl2001}, and activity interventions \cite{MooijHeskes_UAI_13}.
The two-sample test can also be implemented as an independence test that tests for the independence of $Y$ and $I_X$, the indicator
variable that has value $0$ for observational samples and $1$ for samples from the interventional 
distribution in which $X$ has been intervened upon.

\section{Scoring causal predictions}
The constrained minimization in \eref{minimization} may produce several optimal solutions, because the underlying structure may not be identifiable from the inputs. To address this issue,
we propose to use the loss function \eref{loss} and score the confidence of a feature $f$ (e.g., an ancestral relation $X \causes Y$) as:
\begin{equation}\label{eq:confidence}
  C(f) = \min_{W \in \mathcal{W}} \loss(W;I \cup \{(\lnot f,\infty)\}) - \min_{W \in \mathcal{W}} \loss(W;I \cup \{(f,\infty)\}).
\end{equation}
Without going into details here, we note that the confidence \eref{eq:confidence} can be interpreted as a MAP approximation of the log-odds ratio of the probability that feature $f$ is true in a Markov Logic model:
$$\frac{\Prb(f \given I, \mathcal{R})}{\Prb(\lnot f \given I, \mathcal{R})} = \frac{\sum_{W \in \mathcal{W}} e^{-\loss(W;I)} 1_{W \cup \mathcal{R} \models f}}{\sum_{W \in \mathcal{W}} e^{-\loss(W;I)} 1_{W \cup \mathcal{R} \models \lnot f}} \approx \frac{\max_{W \in \mathcal{W}} e^{-\loss(W;I \cup \{(f,\infty)\})}}{\max_{W \in \mathcal{W}} e^{-\loss(W;I \cup \{(\lnot f,\infty)\})}} = e^{C(f)}.$$
In this paper, we usually consider the features $f$ to be ancestral relations, but the idea is more generally applicable. For example, combined with HEJ it can be used to score direct causal relations.

\paragraph{Soundness and completeness}
Our scoring method is sound for oracle inputs:
\begin{theorem}\label{eq:soundness}
Let $\mathcal{R}$ be sound (not necessarily complete) causal reasoning rules. For any feature $f$, the confidence score 
$C(f)$ of \eref{eq:confidence} is sound for oracle inputs with infinite weights.
\end{theorem}
Here, soundness means that $C(f)=\infty$ if $f$ is identifiable from the inputs, 
  $C(f)=-\infty$ if $\lnot f$ is identifiable from the inputs, and $C(f)=0$ otherwise (neither are identifiable).
As features, we can consider for example ancestral relations $f = X \causes Y$ for variables $X,Y$.
We conjecture that the rules \eref{eq:reflexivity}--\eref{eq:lastrule} are ``order-1-complete'', i.e., they allow one to deduce all (non)ancestral relations that are identifiable from oracle conditional independences of order $\leq 1$ in observational data. 
For higher-order inputs additional rules can be derived. However, our primary interest in this work is improving computation time and accuracy, and we are willing to sacrifice completeness. A more detailed study of the completeness properties is left as future work.

\paragraph{Asymptotic consistency}
Denote the number of samples by $N$.
For the frequentist weights in \eref{eq:weightFreq},
we assume that the statistical tests are consistent in the following sense:
\begin{equation}\label{eq:weightFreqConsistency}
\log p_N - \log \alpha_N \xto{P} \begin{cases}
  -\infty & H_1 \\
  +\infty & H_0,
\end{cases}
\end{equation}
as $N \to \infty$, where the null hypothesis $H_0$ is independence/nonancestral relation and the alternative hypothesis $H_1$ is dependence/ancestral relation. Note that we need 
to choose a sample-size dependent threshold $\alpha_N$ such that $\alpha_N \to 0$ at a suitable rate. Kalisch and B{\"u}hlmann \cite{KalischBuehlmann2007} show how this can be done for partial correlation tests under the assumption that the distribution is multivariate Gaussian.

For the Bayesian weighting scheme in \eref{eq:weightBayes}, we assume that for $N\to\infty$,
\begin{equation}\label{eq:weightBayesConsistency}
  w_N \xto{P} \begin{cases}
    -\infty &\text{ if $i$ is true} \\
    +\infty &\text{ if $i$ is false}.
  \end{cases}
\end{equation}
This will hold (as long as there is no model
misspecification) under mild technical conditions for finite-dimensional exponential family models. 
In both cases, the probability of a type I or type II error will converge to 0,
and in addition, the corresponding weight will converge to $\infty$.

\begin{theorem}
Let $\mathcal{R}$ be sound (not necessarily complete) causal reasoning rules. 
For any feature $f$, the confidence score $C(f)$ of \eref{eq:confidence} is asymptotically consistent
under assumption \eref{eq:weightFreqConsistency} or \eref{eq:weightBayesConsistency}.
\end{theorem}
Here, ``asymptotically consistent'' means that the confidence score  
$C(f) \to \infty$ in probability if $f$ is identifiably true, $C(f) \to -\infty$ in probability 
if $f$ is identifiably false, and $C(f)\to 0$ in probability otherwise.

\section{Evaluation} \label{secEval}
In this section we report evaluations on synthetically generated data and an application on a real dataset. Crucially, in causal discovery precision is often more important than recall. In many real-world applications, discovering a few high-confidence causal relations is more useful than finding every possible causal relation, as reflected in recently proposed algorithms, e.g., \cite{ICP}. 

\begin{figure}
  \begin{minipage}[b]{0.6\textwidth}
    \centering%
    \begin{tabular}{ |l|l|l|l|l|l| }%
    \hline%
      \multicolumn{6}{|c|}{\textbf{Average execution time (s)}} \\%
      \hline \hline
    $n$ &$c$& \name\ & HEJ & BAFCI & BACFCI\\\hline
    6 & 1 & 0.21 &	12.09  & 8.39 & 12.51\\ 
    6 & 4 & 1.66 & 432.67 & 11.10 & 16.36\\ \hline
    7 & 1 & 1.03 &	715.74 & 9.37 & 15.12\\ \hline
    8 & 1 & 9.74 &  $\geq 2500 $ & 13.71 & 21.71\\ \hline
    9 & 1 & 146.66 &  $\gg 2500 $ & 18.28 & 28.51 \\\hline
    \end{tabular}
    \caption*{(a)\label{extime}}
    \par\vspace{0pt}
  \end{minipage}
  \centering
  \begin{minipage}[b]{0.35\textwidth}
    \centering     %
	\includegraphics[width=\textwidth]{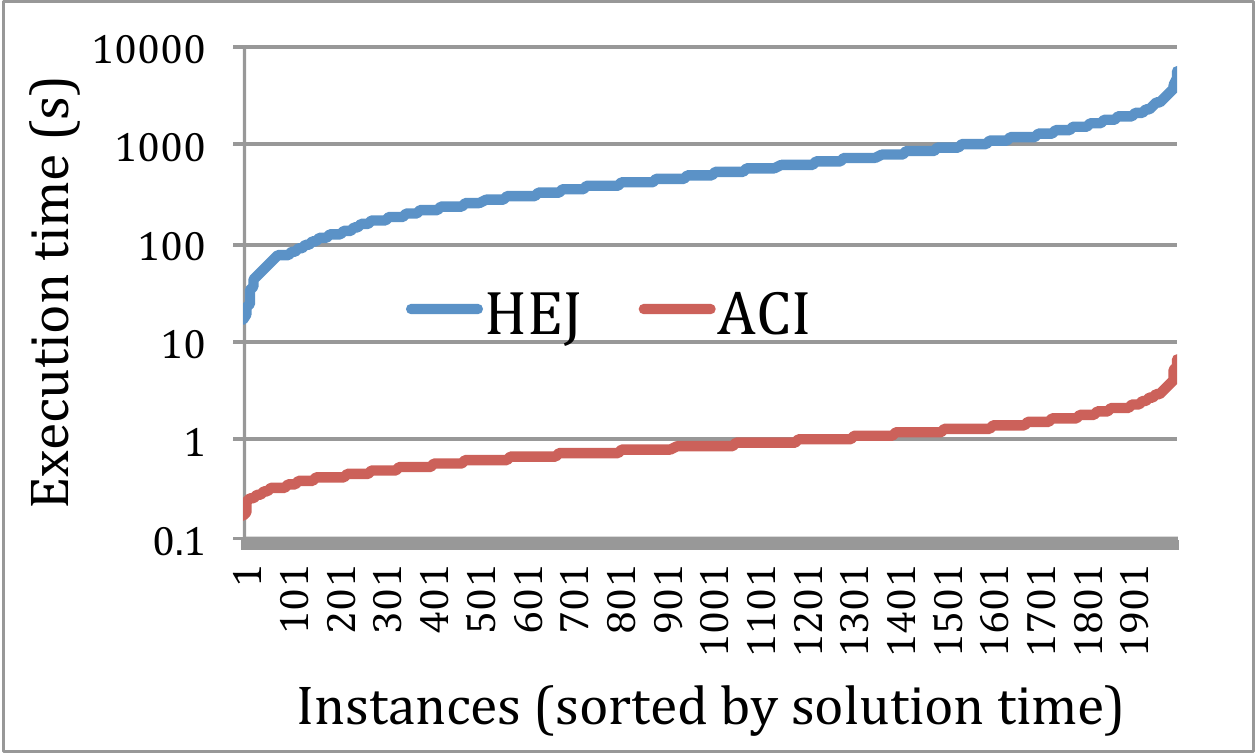}
	\caption*{(b)\label{fig:7_1_time}}
    \par\vspace{0pt}
  \end{minipage}%
  \caption{Execution time comparison on synthetic data for the frequentist test on 2000 synthetic models: (a) average execution time for different combinations of number of variables $n$ and max.\ order $c$; (b) detailed plot of execution times for $n=7, c=1$ (logarithmic scale). \label{fig:times}}
\end{figure}

\paragraph{Compared methods} 
We compare the predictions of \name\ and of the acyclic causally insufficient version of HEJ \cite{antti}, when used in combination with our scoring method \eref{eq:confidence}.
We also evaluate two standard methods: Anytime FCI \cite{Spirtes01ananytime,Zhang:2008:COR:1414091.1414237} and Anytime CFCI \cite{Colombo++2012}, as implemented in the \texttt{pcalg} R package \cite{pcalg}. We use the anytime versions of (C)FCI because they allow for independence test results up to a certain order.
We obtain the ancestral relations from the output PAG using Theorem 3.1 from \cite{ancestralFCI}.
(Anytime) FCI and CFCI do not rank their predictions, but only predict the type of relation: ancestral (which we convert to +1), non-ancestral (-1) and unknown (0).
To get a scoring of the predictions, we also compare with bootstrapped versions of Anytime FCI and Anytime CFCI. We perform the bootstrap by repeating the following procedure 100 times: sample randomly half of the data, perform the independence tests, run Anytime (C)FCI. From the 100 output PAGs we extract the ancestral predictions and average them.
We refer to these methods as BA(C)FCI.
For a fair comparison, we use the same independence tests and thresholds for all methods.

\paragraph{Synthetic data}
We simulate the data using the simulator from HEJ \cite{antti}:
for each experimental condition (e.g., a given number of variables $n$ and order $c$), we generate randomly $M$ linear acyclic models with latent variables and Gaussian noise and sample $N=500$ data points. We then perform independence tests up to order $c$ and weight the (in)dependence statements using the weighting schemes described in Section~\ref{aci}.
For the frequentist weights
we use tests based on partial correlations and Fisher's $z$-transform to obtain approximate $p$-values
(see, e.g., \cite{KalischBuehlmann2007}) with significance level $\alpha=0.05$.
For the Bayesian weights, we use the Bayesian test for conditional independence presented in \cite{DBLP:journals/ci/MargaritisB09} as implemented by HEJ with a prior probability of 0.1 for independence. 

\begin{figure*}[th!]
\centering     %
\subfigure[PR ancestral: n=6]{\label{fig:6_pos}
\includegraphics[width=0.32\textwidth]{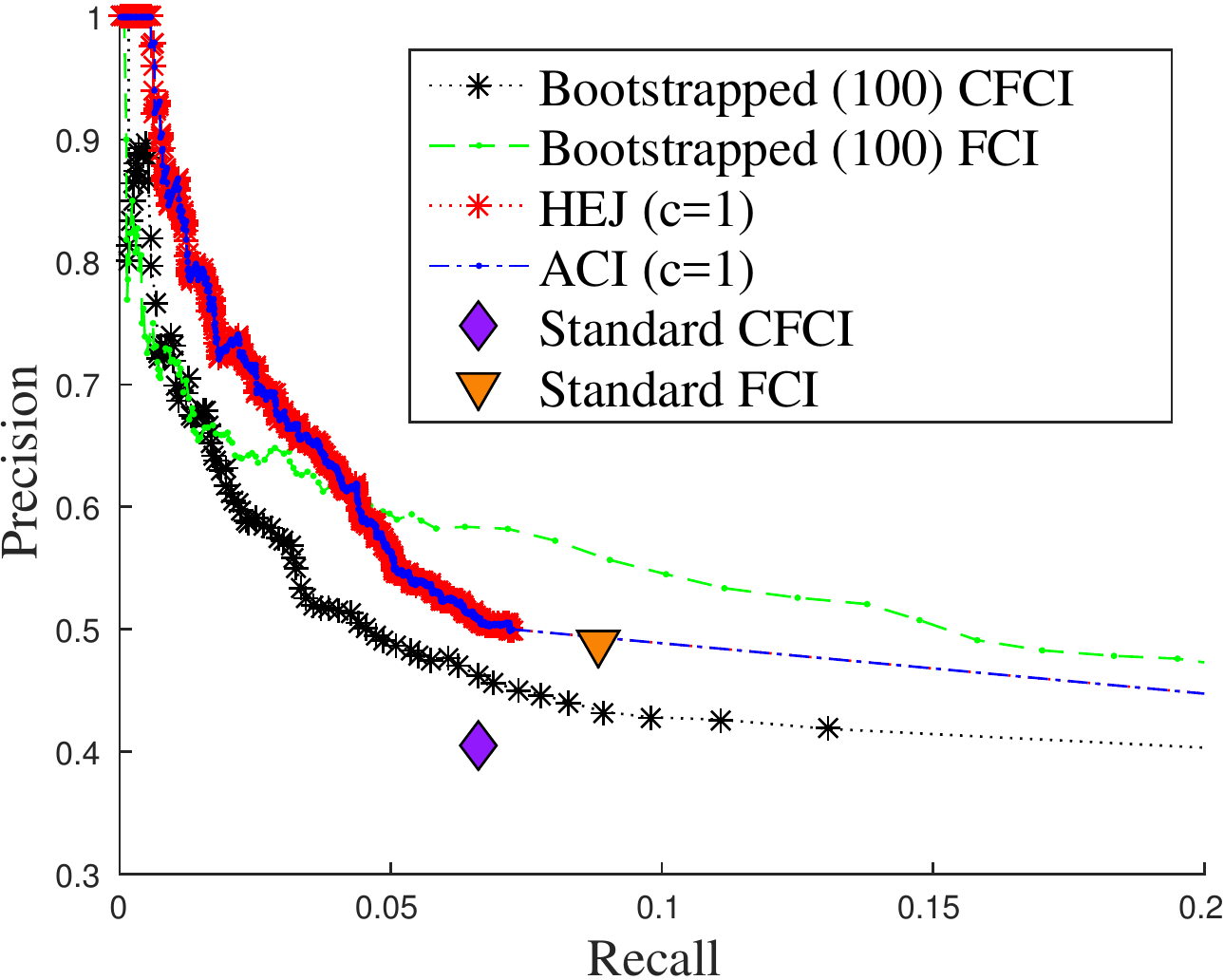}}
\subfigure[PR ancestral: n=6 (zoom)]{\label{fig:6_pos_zoom}\includegraphics[width=0.32\textwidth]{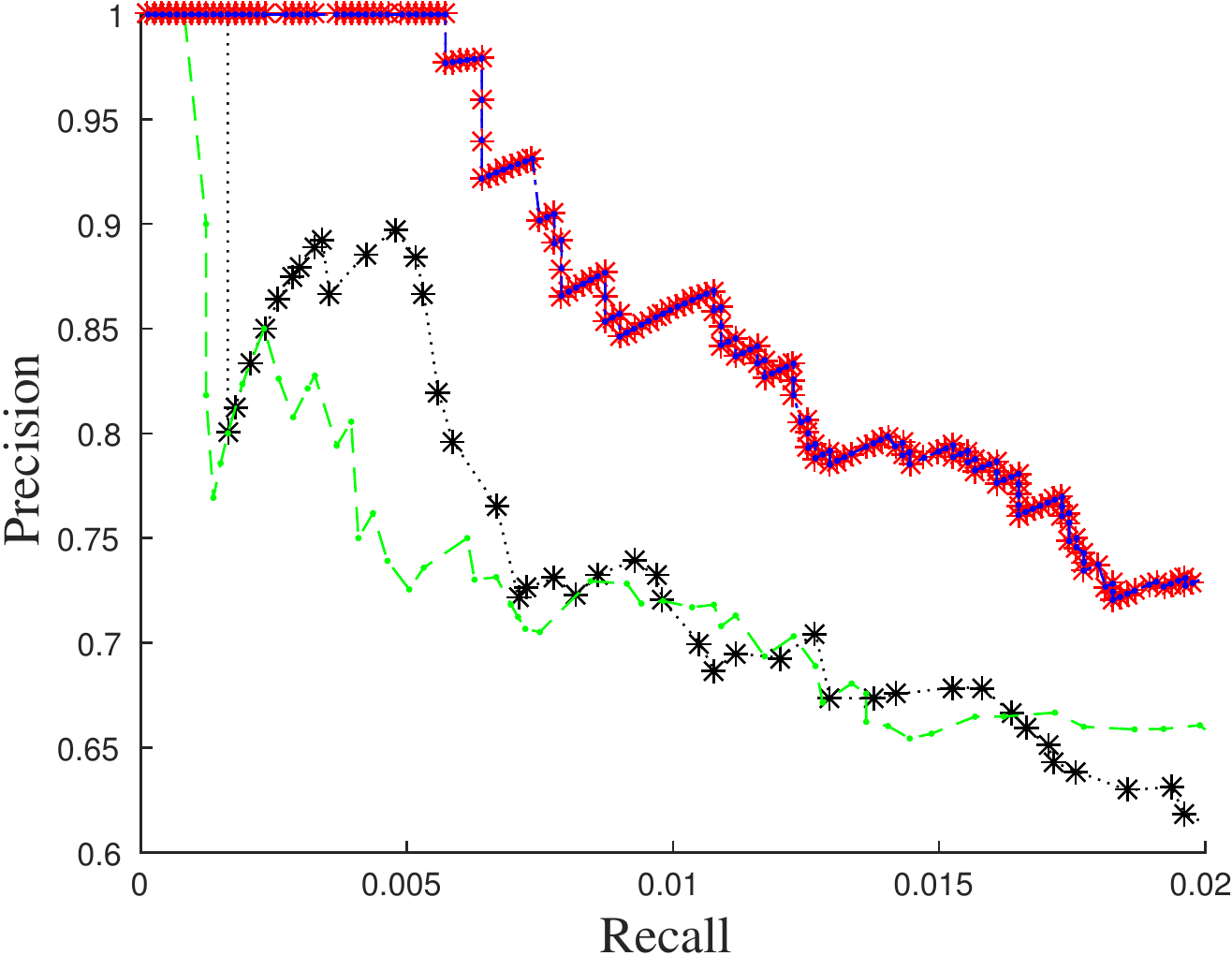}}
\subfigure[PR nonancestral: n=6]{\label{fig:6_neg}
\includegraphics[width=0.32\textwidth]{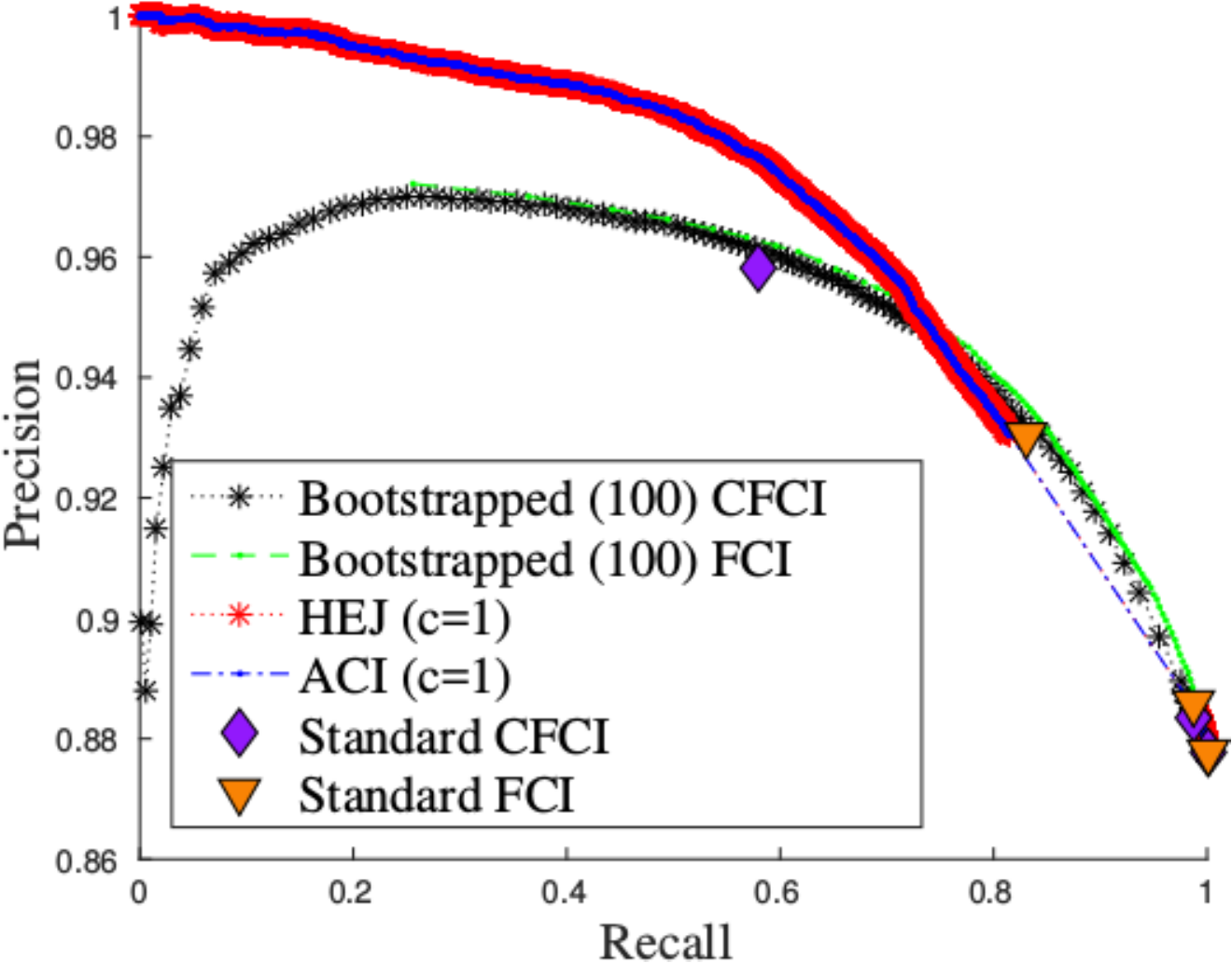}}
\\
\subfigure[PR ancestral: n=8]{\label{fig:8_pos}
\includegraphics[width=0.32\textwidth]{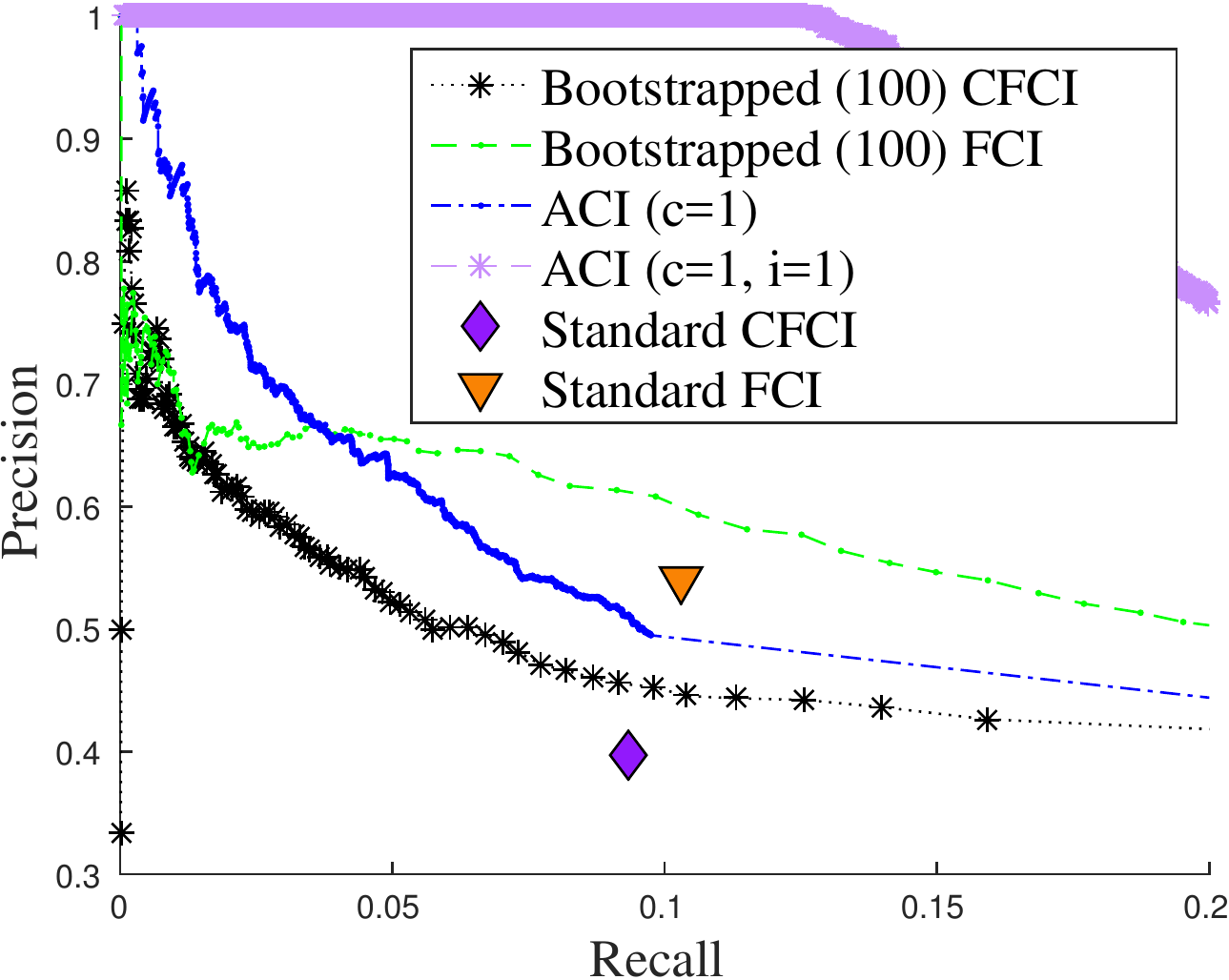}}
\subfigure[PR ancestral: n=8 (zoom)]{\label{fig:8_pos_zoom}\includegraphics[width=0.32\textwidth]{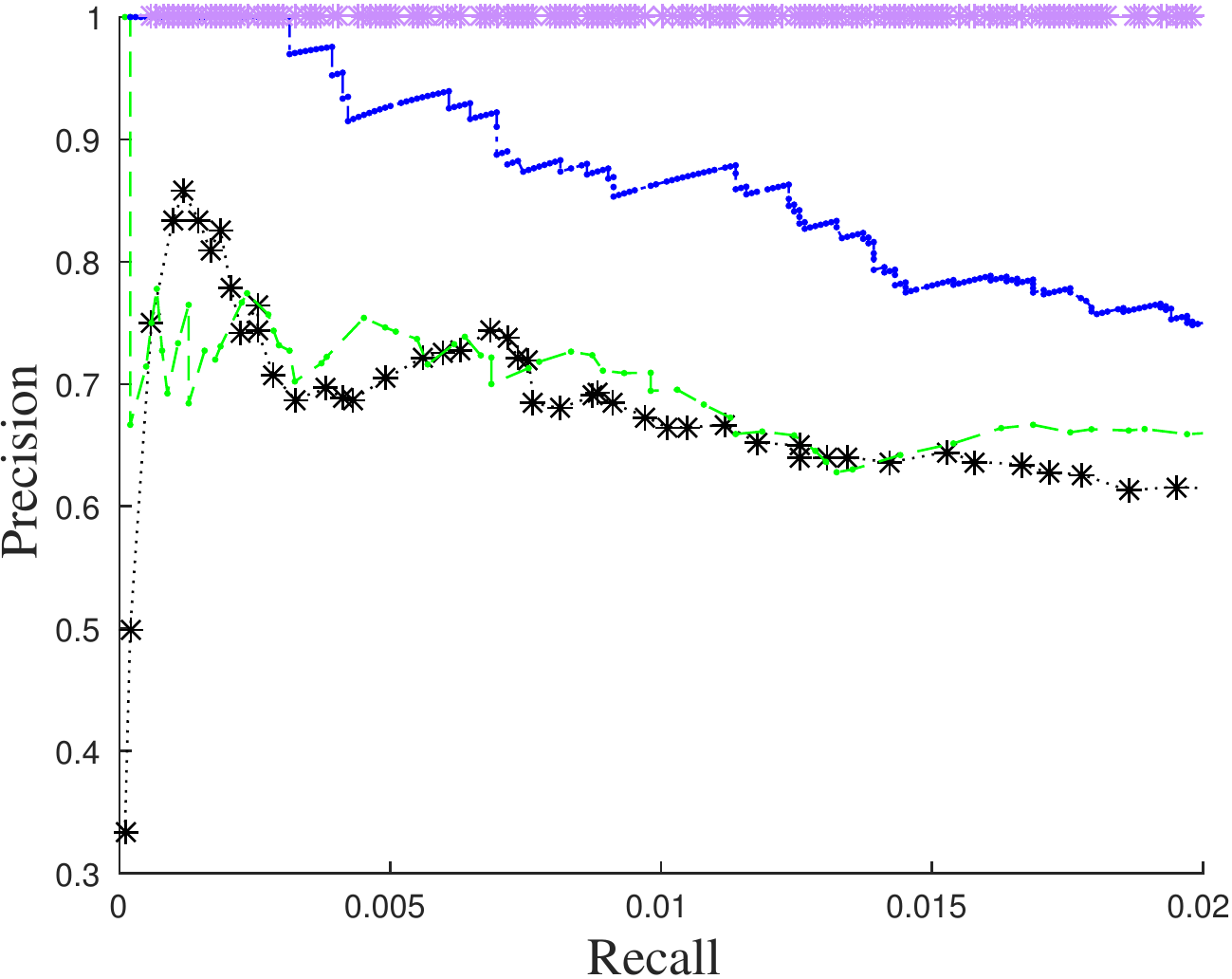}}
\subfigure[PR nonancestral: n=8]{\label{fig:8_neg}
\includegraphics[width=0.32\textwidth]{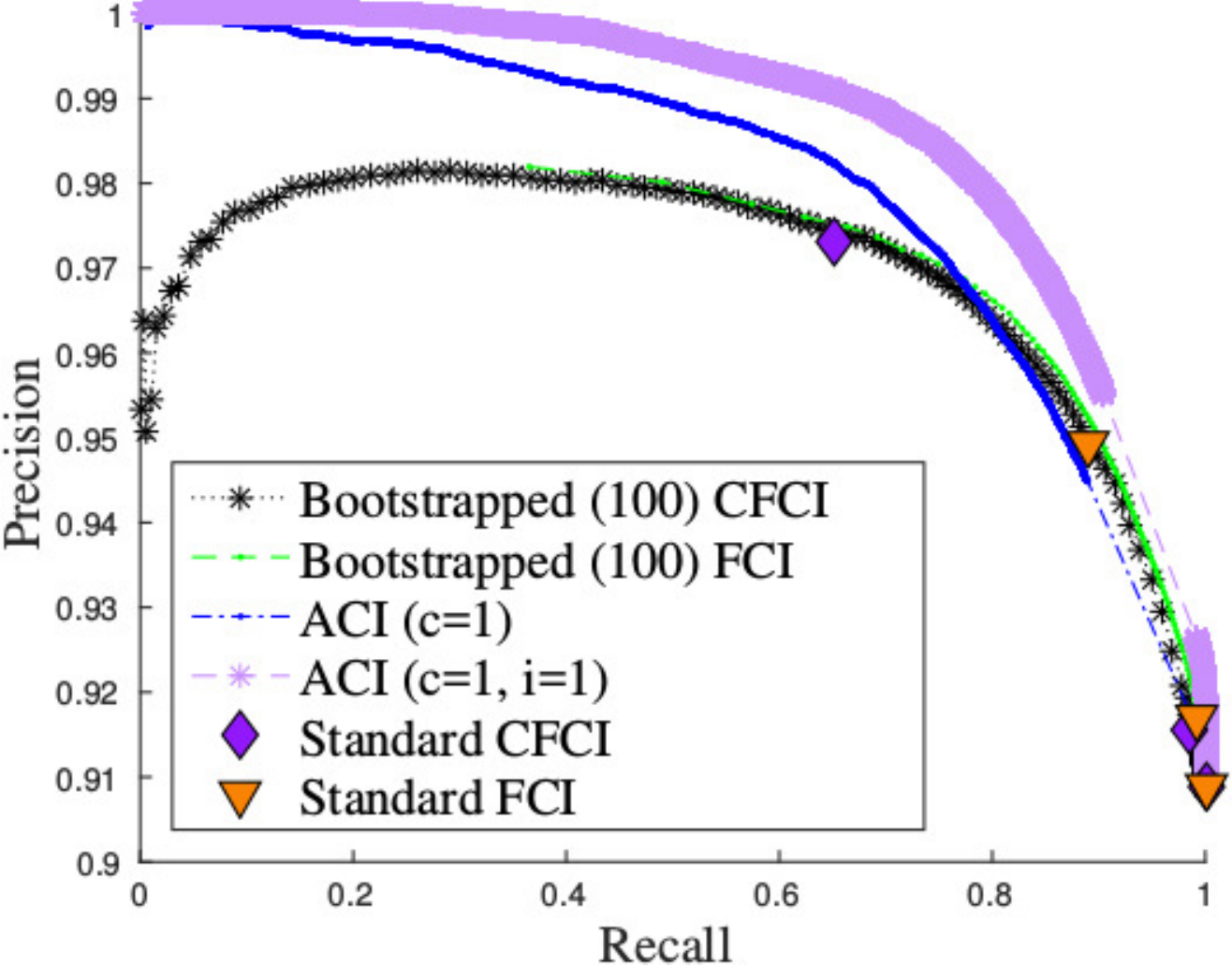}}
\caption{Accuracy on synthetic data for the two prediction tasks (ancestral and nonancestral relations) using the frequentist test with $\alpha=0.05$. The left column shows the precision-recall curve for ancestral predictions, the middle column shows a zoomed-in version in the interval (0,0.02), while the right column shows the nonancestral predictions.}
\label{fig:accuracy}
\end{figure*}

In Figure \ref{fig:times}(a) we show the average execution times on a single core of a 2.80GHz CPU for different combinations of $n$ and $c$, while in Figure \ref{fig:times}(b) we show the execution times for $n=7, c=1$, sorting the execution times in ascending order.
For 7 variables \name\ is almost 3 orders of magnitude faster than HEJ, and the difference grows exponentially as $n$ increases. For 8 variables HEJ can complete only four of the first 40 simulated models before the timeout of $2500$s. For reference we add the execution time for bootstrapped anytime FCI and CFCI.

In Figure \ref{fig:accuracy} we show the accuracy of the predictions with precision-recall (PR) curves for both ancestral ($X \causes Y$) and nonancestral ($X \notcauses Y$) relations, in different settings. In this Figure, for \name\ and HEJ all of the results are computed using frequentist weights and, as in all evaluations, our scoring method \eref{eq:confidence}. While for these two methods we use $c=1$, for (bootstrapped) (C)FCI we use all possible independence test results ($c=n-2$). In this case, the anytime versions of FCI and CFCI are equivalent to the standard versions of FCI and CFCI. Since the overall results are similar, we report the results with the Bayesian weights in the Supplementary Material.

In the first row of Figure \ref{fig:accuracy}, we show the setting with $n=6$ variables. The performances of HEJ and ACI coincide, performing significantly better for nonancestral predictions and the top ancestral predictions (see zoomed-in version in Figure \ref{fig:accuracy}(b)). This is remarkable, as HEJ and ACI use only independence test results up to order $c=1$, in contrast with (C)FCI which uses independence test results of all orders. Interestingly, the two discrete optimization algorithms do not seem to benefit much from higher order independence tests, thus we omit them from the plots (although we add the graphs in the Supplementary Material). Instead, bootstrapping traditional methods, oblivious to the (in)dependence weights, seems to produce surprisingly good results. Nevertheless, both ACI and HEJ outperform bootstrapped FCI and CFCI, suggesting these methods achieve nontrivial error-correction.

In the second row of Figure \ref{fig:accuracy}, we show the setting with 8 variables. In this setting HEJ is too slow. In addition to the previous plot, we plot the accuracy of \name\ when there is oracle background knowledge on the descendants of one variable ($i=1$). This setting simulates the effect of using interventional data, and we can see that the performance of  \name\ improves significantly, especially in the ancestral preditions. The performance of (bootstrapped) FCI and CFCI is limited by the fact that they cannot take advantage of this background knowledge, except with complicated postprocessing \cite{borboudakis2012}.

\paragraph{Application on real data}

We consider the challenging task of reconstructing
a signalling network from flow cytometry data \cite{SPP05} under different experimental conditions. 
Here we consider one experimental condition
as the observational setting and seven others as 
interventional settings. More details and more evaluations are reported in the Supplementary Material.
In contrast to likelihood-based approaches 
like \cite{SPP05,EatonMurphy07,MooijHeskes_UAI_13,Rothenhausler++2015}, in our
approach we do not need to model the interventions quantitatively. We only
need to know the intervention \emph{targets}, while the intervention \emph{types} do not matter.
Another advantage of our approach is that it takes into account possible latent variables.

\begin{figure}[t!]
\centering     %
\subfigure[Bootstrapped (100) anytime CFCI (input: independences of order $\leq 1$)]{
\includegraphics[width=0.27\textwidth]{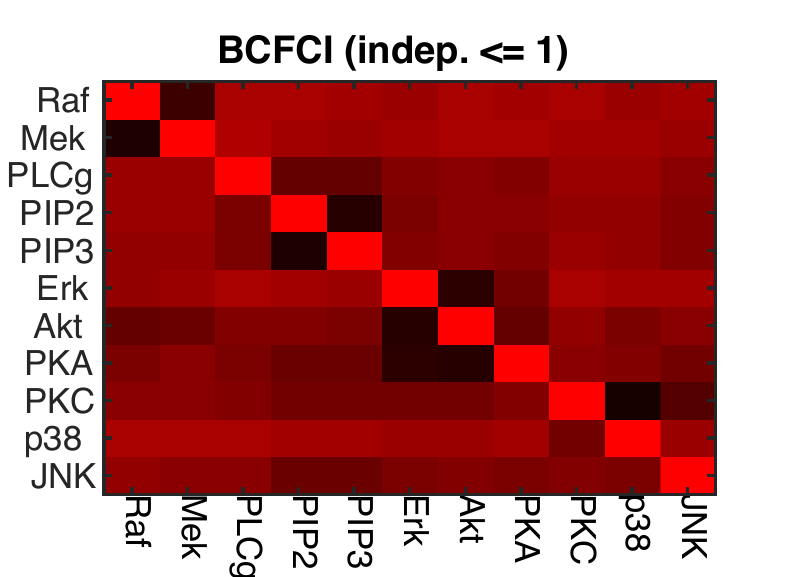}}\qquad
\subfigure[ACI (input: weighted ancestral relations)]{
\includegraphics[width=0.27\textwidth]{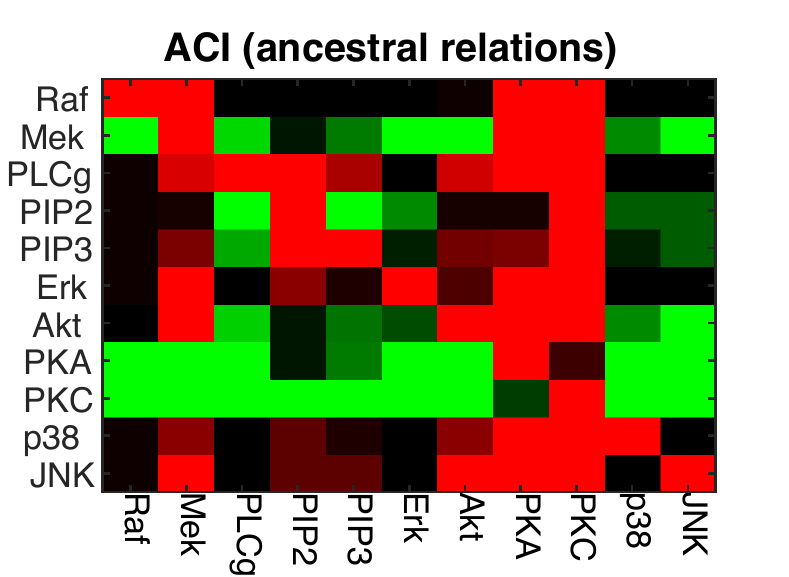}}\qquad
\subfigure[ACI (input: independences of order $\leq 1$, weighted ancestral relations)]{\includegraphics[width=0.27\textwidth]{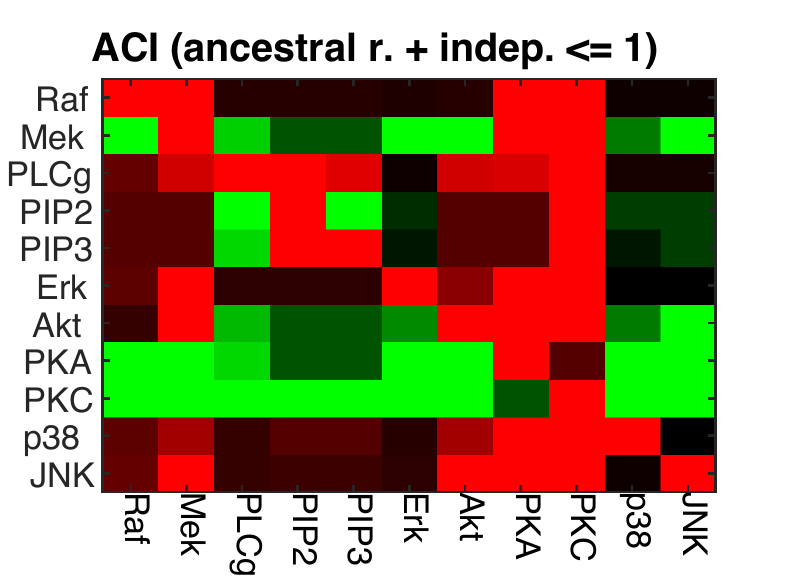}}
\subfigure{\includegraphics[width=0.07\textwidth]{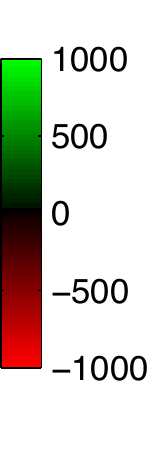}}
  \caption{\label{fig:Sachs}Results for flow cytometry dataset. Each matrix represents the ancestral relations, where each row represents a cause and each column an effect. The colors encode the confidence levels: green is positive, black is unknown,  while red is negative. The intensity of the color represents the degree of confidence. For example, ACI identifies MEK to be a cause of RAF with high confidence.}
\end{figure}

We use a $t$-test to test for each intervention and for each variable whether its
distribution changes with respect to the observational condition.
We use the $p$-values of these tests as in \eref{eq:weightFreq} in order to obtain 
weighted ancestral relations that are used as input (with threshold $\alpha = 0.05$). 
For example, if adding U0126 (a MEK inhibitor) changes the distribution of RAF 
significantly with respect to the observational baseline, we get a
weighted ancestral relation MEK$\causes$RAF. In addition, we use partial correlations
up to order 1 (tested in the observational data only) to obtain weighted independences used as input. 
We use \name\ with \eref{eq:confidence} to score the ancestral relations for each ordered pair of variables. 
The main results are illustrated in Figure~\ref{fig:Sachs}, where we compare ACI with bootstrapped anytime CFCI under different inputs. The output for boostrapped anytime FCI is similar, so we report it only in the Supplementary Material. 
Algorithms like (anytime) (C)FCI can only use the independences in the observational data as input and therefore miss the strongest signal, \emph{weighted ancestral relations}, which are obtained by comparing interventional with observational data.
In the Supplementary Material, we compare also with other methods (\cite{ICP}, \cite{MooijHeskes_UAI_13}). Interestingly, as we show there, our results are similar to the best acyclic model reconstructed by the score-based method from \cite{MooijHeskes_UAI_13}. As for other constraint-based methods, HEJ is computationally unfeasible in this setting, while COMBINE assumes perfect interventions (while this dataset contains mostly activity interventions). 

Notably, our algorithms can correctly recover from faithfulness violations (e.g., the independence between MEK and ERK), because they take into account the weight of the input statements (the weight of the independence is considerably smaller than that of the ancestral relation, which corresponds with a quite significant change in distribution). 
In contrast, methods that start by reconstructing the skeleton, like (anytime) (C)FCI, would decide that MEK and ERK are nonadjacent, and are unable to recover from that 
erroneous decision. This illustrates another advantage of our approach.

\section{Discussion and conclusions}
As we have shown, ancestral structures are very well-suited for causal discovery. They offer a natural way to incorporate background causal knowledge, e.g., from  experimental data, and allow a huge computational advantage over existing representations for error-correcting algorithms, such as \cite{antti}. When needed, ancestral structures can be mapped to a finer-grained representation with direct causal relations, as we sketch in the Supplementary Material.
Furthermore, confidence estimates on causal predictions are extremely helpful in practice, and can significantly boost the reliability of the output.
Although standard methods, like bootstrapping (C)FCI, already provide reasonable estimates, methods that take into account the confidence in the inputs, as the one presented here, can lead to further improvements of the reliability of causal relations inferred from data.

Strangely (or fortunately) enough, neither of the optimization methods seems to improve much with higher order independence test results. 
We conjecture that this may happen because our loss function essentially assumes that the test results are independent from another (which is not true). Finding a way to take this into account in the loss function may further improve the achievable accuracy, but such an extension may not be straightforward. 

\subsubsection*{Acknowledgments}
SM and JMM were supported by NWO, the Netherlands Organization for Scientific Research (VIDI grant 639.072.410). SM was also supported by the Dutch programme COMMIT/ under the Data2Semantics project.
TC was supported by NWO grant 612.001.202 (MoCoCaDi), and EU-FP7 grant agreement n.603016 (MATRICS).
We also thank Sofia Triantafillou for her feedback, especially for pointing out the correct way to read ancestral relations from a PAG.

\section{Proofs}

\subsection{ACI causal reasoning rules}

We give a combined proof of all the ACI reasoning rules. Note that the numbering of the rules here is different from the numbering used in the main paper.
\begin{lemma}
For $X$, $Y$, $Z$, $U$, $\B{W}$ disjoint (sets of) variables:
  \begin{enumerate}
    \item $(X \CI Y \mid \B{W}) \wedge (X \notcauses \B{W}) \implies X \notcauses Y$
    \item $X \nCI Y \mid \B{W} \cup [Z] \implies (X \nCI Z \mid \B{W}) \land (Z \notcauses \{X,Y\} \cup \B{W})$
    \item $X \CI Y \mid \B{W} \cup [Z] \implies (X \nCI Z \mid \B{W}) \land (Z \causes \{X,Y\} \cup \B{W})$
    \item $(X \CI Y \mid \B{W} \cup [Z]) \land (X \CI Z \mid \B{W} \cup U) \implies (X \CI Y \mid \B{W} \cup U)$
    \item $(Z \nCI X \mid \B{W}) \land (Z \nCI Y \mid \B{W}) \land (X \CI Y \mid \B{W}) \implies X \nCI Y \mid \B{W} \cup Z$
  \end{enumerate}
\end{lemma}
\begin{proof}
  We assume a causal DAG with possible latent variables, the causal Markov assumption, and the causal faithfulness assumption.
\begin{enumerate}
  \item This is a strengthened version of rule $\mathcal{R}2$(i) in \cite{conf/aistats/EntnerHS13}: note that the additional assumptions made there ($Y \notcauses \B{W}$, $Y \notcauses X$) are redundant and not actually used in their proof. For completeness, we give the proof here. If $X \causes Y$, then there is a directed path from $X$ to $Y$. As all paths between $X$ and $Y$ are blocked by $\mathbf{W}$, the directed path from $X$ to $Y$ must contain a node $W \in \B{W}$. Hence $X \causes W$, a contradiction with $X \notcauses \mathbf{W}$.
  \item If $X \nCI Y \mid \B{W} \cup [Z]$ then there exists a path $\pi$ between $X$ and $Y$ such that each
    noncollider on $\pi$ is not in $\B{W} \cup \{Z\}$, every collider on $\pi$ is ancestor of $\B{W} \cup \{Z\}$,
    and there exists a collider on $\pi$ that is ancestor of $Z$ but not of $\B{W}$. Let $C$ be the collider on $\pi$
    closest to $X$ that is ancestor of $Z$ but not of $\B{W}$. Note that
    \begin{enumerate}
      \item The path $X \cdots C \to \cdots \to Z$ is d-connected given $\B{W}$.
      \item $Z \notcauses \B{W}$ (because otherwise $C \causes Z \causes \B{W}$, a contradiction).
      \item $Z \notcauses Y$ (because otherwise the path $X \cdots C \to \cdots \to Z \to \cdots \to Y$ would be d-connected given $\B{W}$, a contradiction).
    \end{enumerate}
    Hence we conclude that $X \nCI Z \mid \B{W}$, $Z \notcauses \B{W}$, $Z \notcauses Y$, and by symmetry also $Z \notcauses X$.
  \item Suppose $X \CI Y \mid \B{W} \cup [Z]$. Then there exists a path $\pi$ between $X$ and $Y$, such that each noncollider on $\pi$ is not in $\B{W}$, each collider on $\pi$ is an ancestor of $\B{W}$, and $Z$ is a noncollider on $\pi$. Note that
    \begin{enumerate}
      \item The subpath $X \cdots Z$ must be d-connected given $\B{W}$. 
      \item $Z$ has at least one outgoing edge on $\pi$. Follow this edge further along $\pi$ until reaching either $X$, $Y$, or the first collider. When a collider is reached, follow the directed path to $\B{W}$. Hence there is a directed path from $Z$ to $X$ or $Y$ or to $\B{W}$, i.e., $Z \causes \{X, Y\} \cup \B{W}$.
    \end{enumerate}
  \item If in addition, $X \CI Z \mid \B{W} \cup U$, then $U$ must be a noncollider on the subpath $X \cdots Z$. Therefore, $X \CI Y \mid \B{W} \cup U$.
  \item Assume that $Z \nCI X \mid \B{W}$ and $Z \nCI Y \mid \B{W}$. Then there must be paths $\pi$ between $Z$ and $X$ and $\rho$ between $Z$ and $Y$ such that each noncollider is not in $\B{W}$ and each collider is ancestor of $\B{W}$.  Let $U$ be the node on $\pi$ closest to $X$ that is also on $\rho$ (this could be $Z$). Then we have a path $X \cdots U \cdots Y$ such that each collider (except $U$) is ancestor of $\B{W}$ and each noncollider (except $U$) is not in $\B{W}$. This path must be blocked given $\B{W}$ as $X \CI Y \mid \B{W}$. If $U$ would be a noncollider on this path, it would need to be in $\B{W}$ in order to block it; however, it must then also be a noncollider on $\pi$ or $\rho$ and hence cannot be in $\B{W}$. Therefore, $U$ must be a collider on this path and cannot be ancestor of $\B{W}$. We have to show that $U$ is ancestor of $Z$. If $U$ were a collider on $\pi$ or $\rho$, it would be ancestor of $\B{W}$, a contradiction. Hence $U$ must have an outgoing arrow pointing towards $Z$ on $\pi$ and $\rho$. If we encounter a collider following the directed edges, we get a contradiction, as that collider, and hence $U$, would be ancestor of $\B{W}$. Hence $U$ is ancestor of $Z$, and therefore, $X \nCI Y \mid \B{W} \cup Z$.
\end{enumerate}
\end{proof}

\subsection{Soundness}
\begin{theorem}
Let $\mathcal{R}$ be sound (not necessarily complete) causal reasoning rules.
For any feature $f$, the confidence score 
$C(f)$ of (16) is sound for oracle inputs with infinite weights, i.e.,
$C(f)=\infty$ if $f$ is identifiable from the inputs, $C(f)=-\infty$ if $\lnot f$ is identifiable from the inputs, and $C(f)=0$ otherwise (neither are identifiable).
\end{theorem}
\begin{proof}
We assume that the data generating process is described by a causal DAG which may
contain additional latent variables, and that the distributions are faithful to the DAG.
The theorem then follows directly from the soundness of the rules and the soundness of logical reasoning.
\end{proof}

\subsection{Asymptotic consistency of scoring method}
\begin{theorem}
Let $\mathcal{R}$ be sound (not necessarily complete) causal reasoning rules.
For any feature $f$, the confidence score $C(f)$ of (16) is asymptotically consistent
under assumption (14) or (15) in the main paper, i.e.,
  \begin{itemize}
    \item $C(f) \to \infty$ in probability if $f$ is identifiably true,
    \item $C(f) \to -\infty$ in probability if $f$ is identifiably false,
    \item $C(f)\to 0$ in probability otherwise (neither are identifiable).
  \end{itemize}
\end{theorem}
\begin{proof}
As the number of statistical tests is fixed (or at least
bounded from above), the probability of \emph{any} error in the test results
converges to 0 asymptotically. The loss function of all structures that 
do not correspond with the properties of the true causal DAG
converges to $+\infty$ in probability, whereas the loss function of all
structures that are compatible with properties of the true causal DAG
converges to 0 in probability. 
\end{proof}

\begin{figure*}[h!]
\centering     %
\subfigure[PR ancestral]{\label{fig:higher_order_pos}
\includegraphics[width=0.65\textwidth]{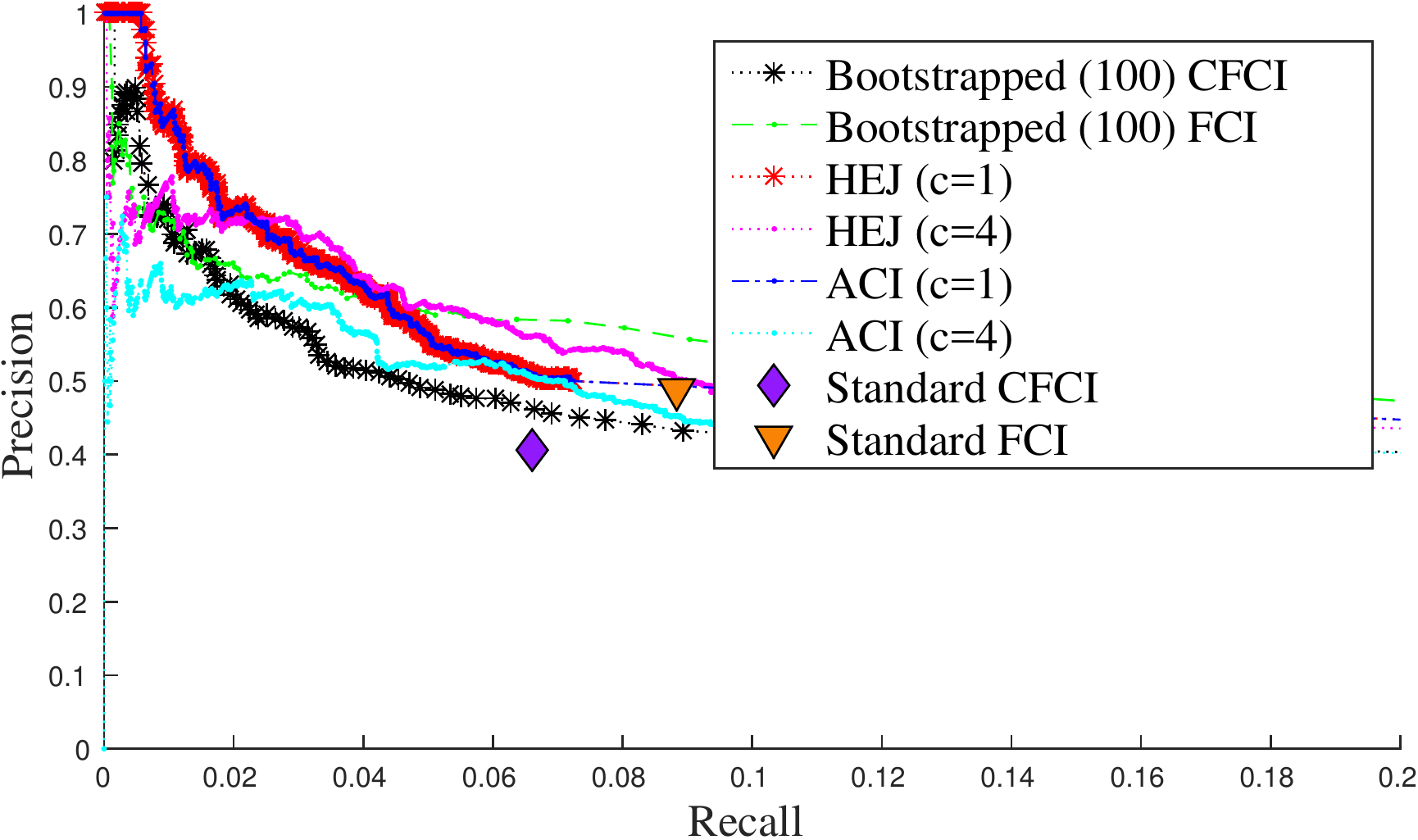}}\qquad
\\
\subfigure[PR ancestral (zoom)]{\label{fig:higher_order_pos_zoom}\includegraphics[width=0.65\textwidth]{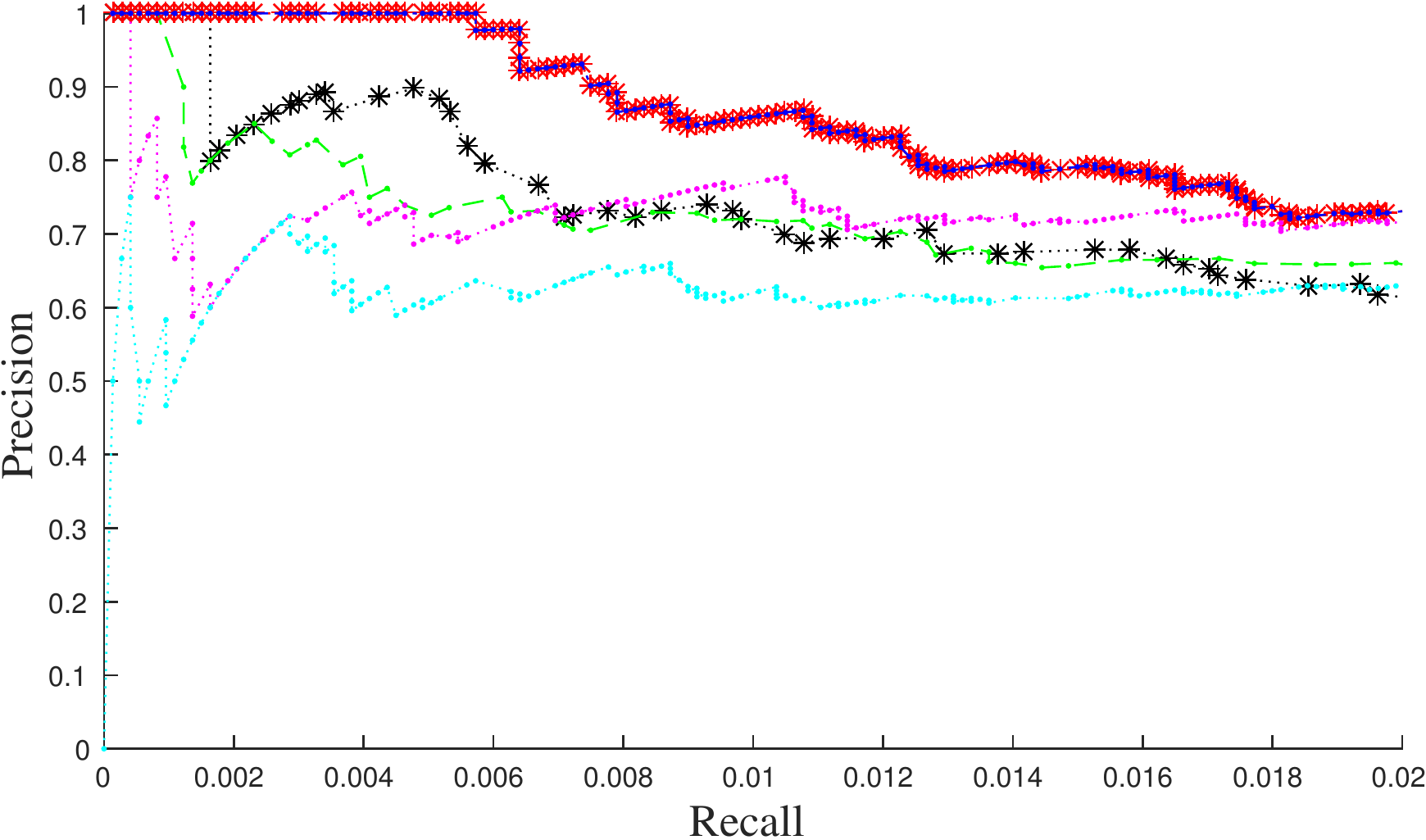}}\qquad
\\
\subfigure[PR nonancestral]{\label{fig:higher_order_neg}
\includegraphics[width=0.65\textwidth]{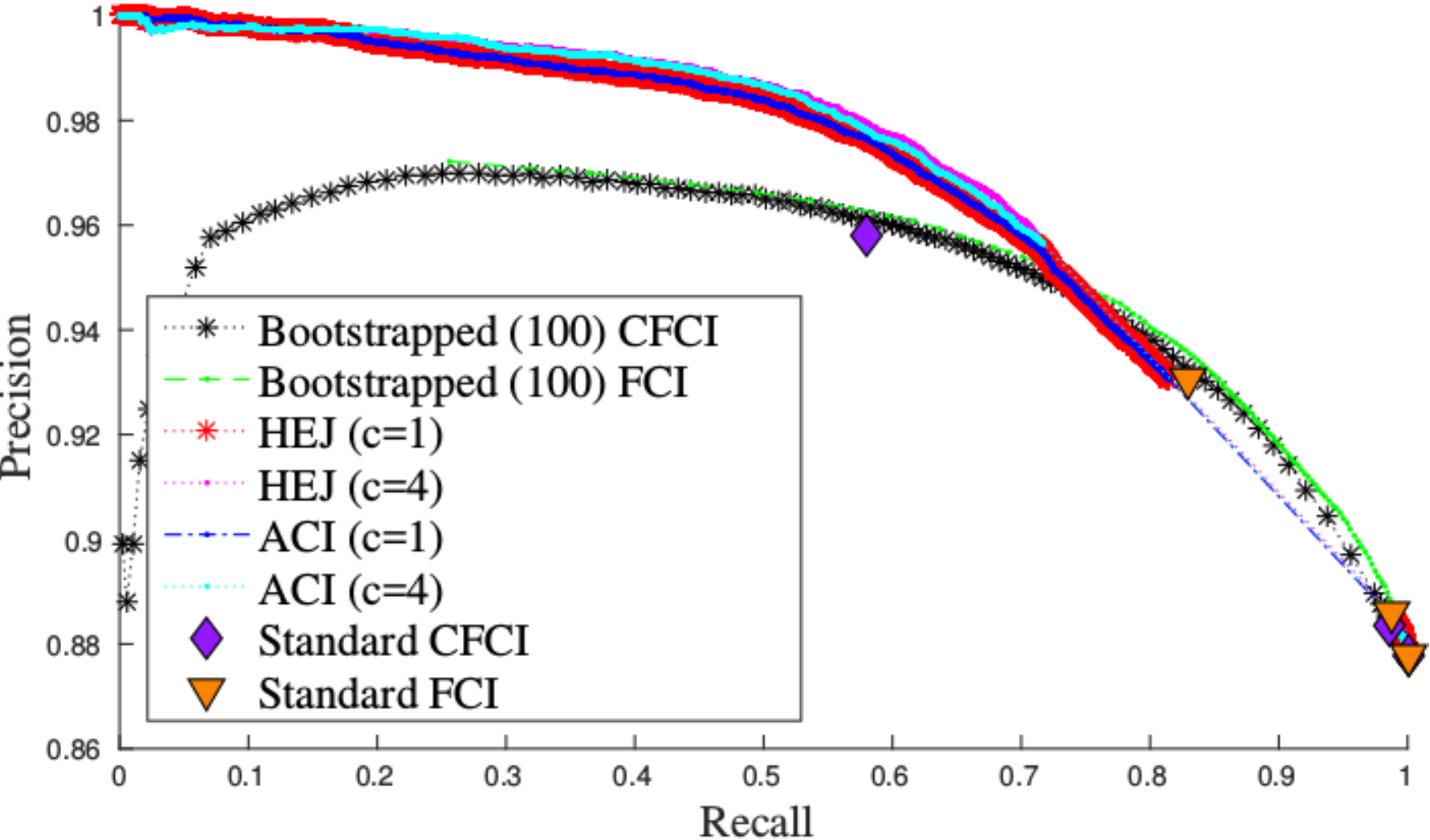}}
\caption{Synthetic data: accuracy for the two prediction tasks (ancestral and nonancestral relations) for $n=6$ variables using the frequentist test with $\alpha=0.05$, also for higher order $c$.
\label{fig:higher_order}}
\end{figure*}

\begin{figure*}[h!]
\centering     %
\subfigure[PR ancestral]{\label{fig:6_bayes_pos}
\includegraphics[width=0.5\textwidth]{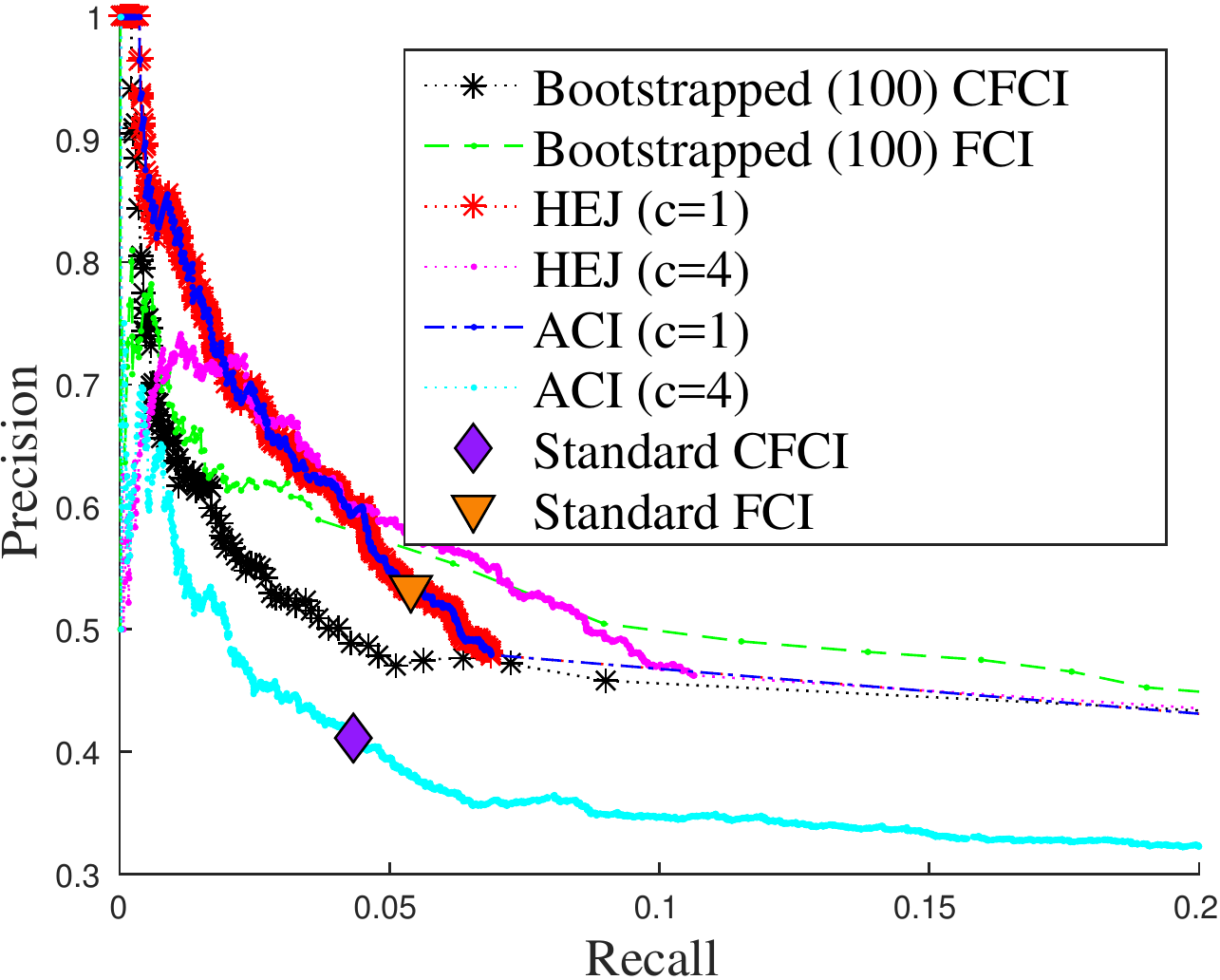}}
\\
\subfigure[PR ancestral (zoom)]{\label{fig:6_bayes_pos_zoom}\includegraphics[width=0.5\textwidth]{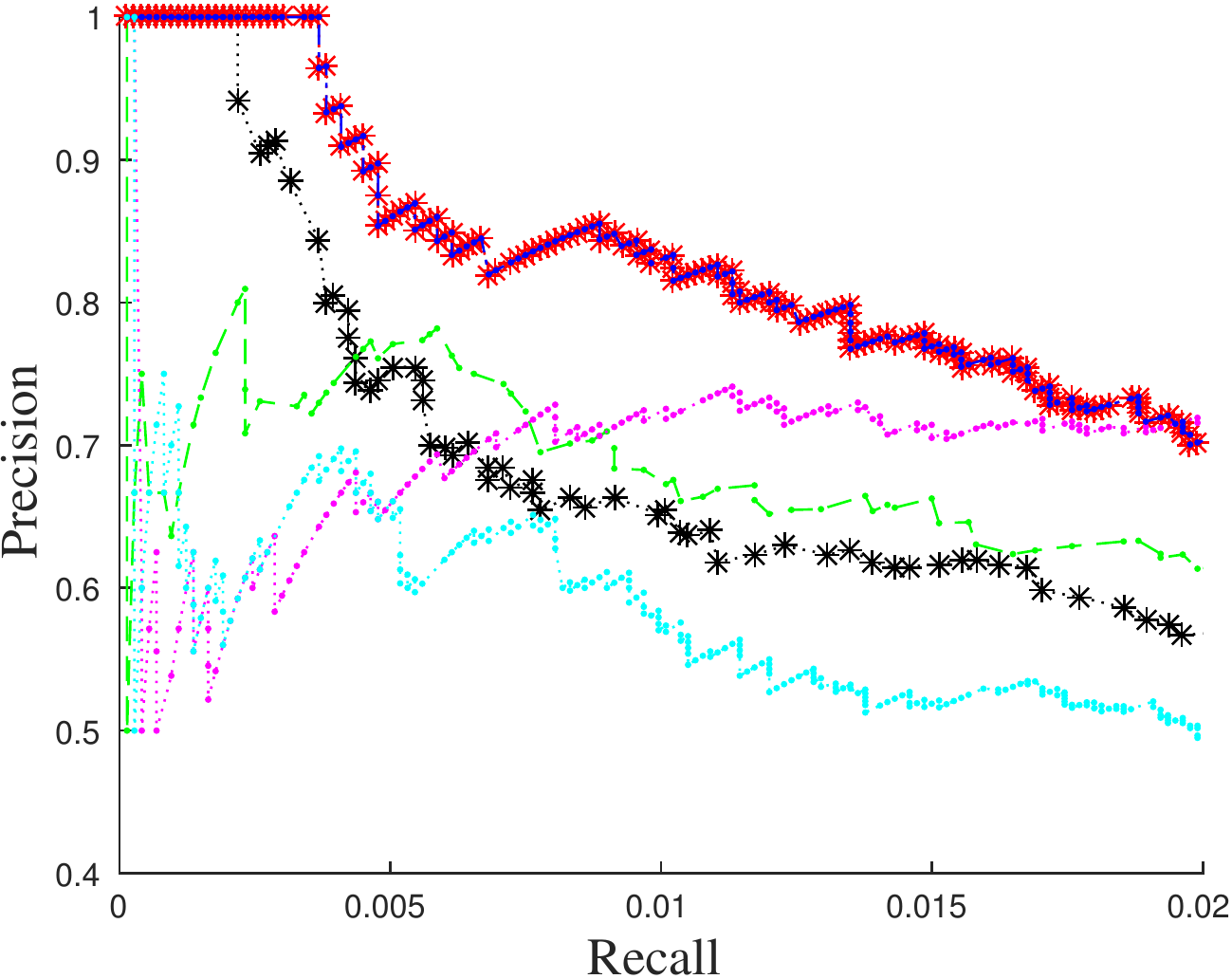}}
\\
\subfigure[PR nonancestral]{\label{fig:6_bayes_neg}
\includegraphics[width=0.51\textwidth]{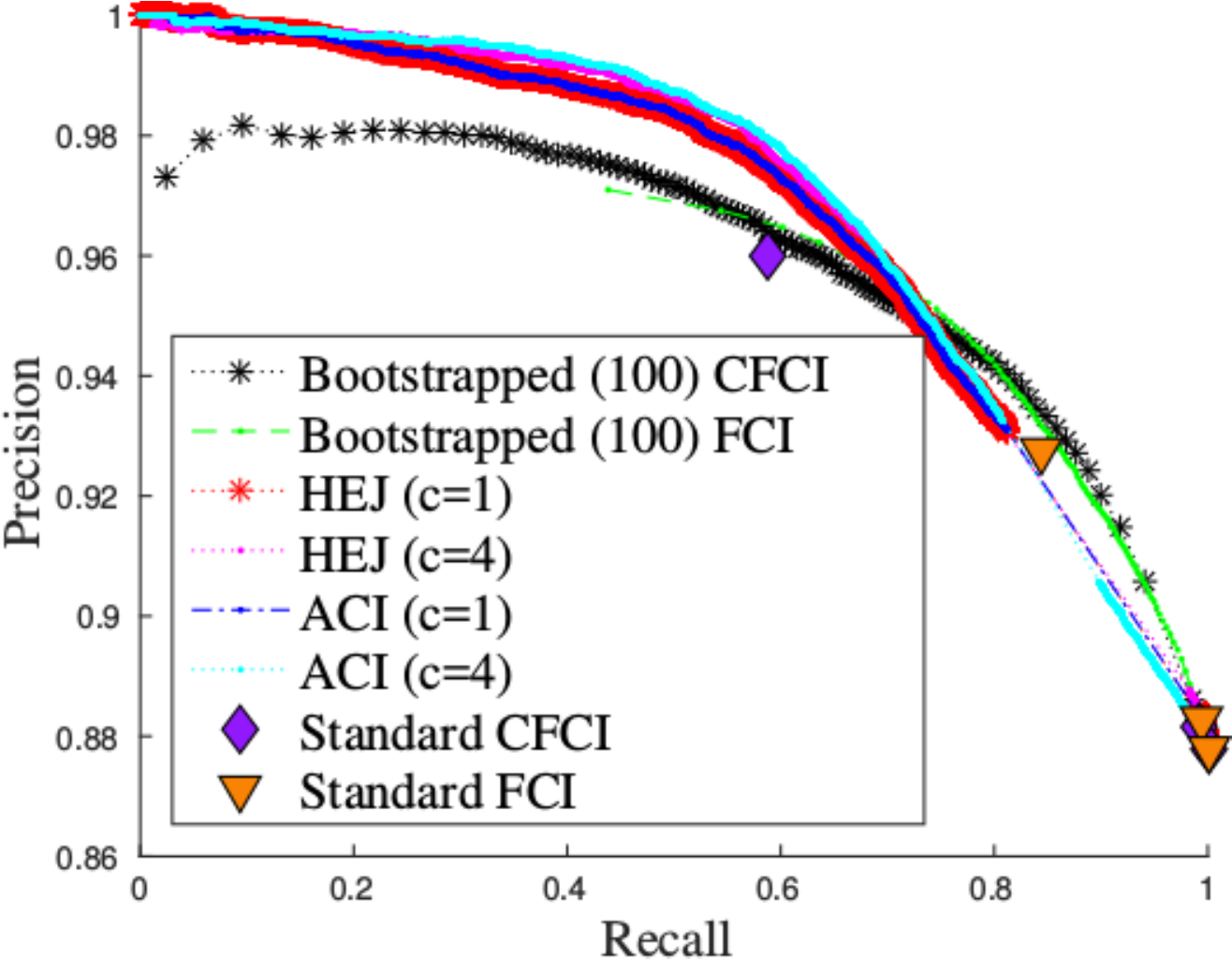}}
  \caption{Synthetic data: accuracy for the two prediction tasks (ancestral and nonancestral relations) for $n=6$ variables using the Bayesian test with prior probability of independence $p=0.1$. \label{fig:bayes}}
\end{figure*}

\section{Additional results on synthetic data}
In Figures \ref{fig:higher_order} and \ref{fig:bayes} we show the performance of ACI and HEJ \cite{antti} for higher order independence test results ($c=4$). As in the main paper, for (bootstrapped) FCI and CFCI we use $c=4$, because it gives the best predictions for these methods. In Figure \ref{fig:higher_order} we report more accuracy results on the frequentist test with $\alpha=0.05$, the same setting as Figure 2 (a-c) in the main paper. As we see, the performances of ACI and HEJ do not really improve with higher order but actually seem to deteriorate. 

In Figure \ref{fig:bayes} we report accuracy results on synthetic data also for the Bayesian test described in the main 
paper, with prior probability of independence $p=0.1$.
Using the Bayesian test does not change the overall conclusions:  ACI and HEJ overlap for order $c=1$ and they perform better than bootstrapped (C)FCI.

\newpage
\section{Application on real data}
We provide more details and more results on the real-world dataset that was briefly described in the main paper, the flow cytometry data \cite{SPP05}. 
The data consists of simultaneous measurements of expression levels of 11 biochemical agents in individual cells of the human immune system under 14 different experimental conditions.

\subsection{Experimental conditions}
The experimental conditions can be grouped into two batches of 8 conditions each that have very similar interventions: 
\begin{itemize}
\item ``no-ICAM'', used in the main paper and commonly used in the literature;
\item ``ICAM'', where Intercellular Adhesion Protein-2 (ICAM-2) was added (except when PMA or $\beta$2CAMP was added).
\end{itemize}
For each batch of 8 conditions, 
the experimenters added $\alpha$-CD3 and $\alpha$-CD28 to activate the signaling network in 6 out of 8 conditions. For the remaining two conditions (PMA and $\beta$2CAMP), $\alpha$-CD3 and $\alpha$-CD28 were not added (and neither was ICAM-2). We can consider the \emph{absence} of these stimuli as a global intervention relative to the observational baseline (where $\alpha$-CD3 and $\alpha$-CD28 are present, and in addition ICAM-2 is present in the ICAM batch).
For each batch (ICAM and no-ICAM), we can consider an observational dataset and 7 interventional datasets with different activators and inhibitors added to the cells, as described in Table \ref{sachsexperiments}. Note that the datasets from the last two conditions are the same in both settings. For more information about intervention types, see \cite{MooijHeskes_UAI_13}.

\begin{table}[t!]
\caption{Reagents used in the various experimental conditions in \cite{SPP05} and corresponding intervention types and targets. The intervention types and targets are based on (our interpretation of) biological background knowledge. The upper table describes the ``no-ICAM'' batch of conditions that is most commonly used in the literature. The lower table describes the additional ``ICAM'' batch of conditions that we also use here.\label{sachsexperiments}}
\medskip
\centerline{no-ICAM: \quad \begin{tabular}{|lll|ll|}
\hline
\multicolumn{3}{|c|}{Reagents} & \multicolumn{2}{c|}{Intervention} \\
\hline
$\alpha$-CD3, $\alpha$-CD28 & ICAM-2 & Additional & Target & Type \\
\hline
+ & - &- & - & (observational) \\
+ & - &AKT inhibitor & AKT & activity \\
+ & - &G0076 & PKC & activity \\
+ & - &Psitectorigenin & PIP2 & abundance \\
+ & - &U0126 & MEK & activity \\
+ & - &LY294002 & PIP2/PIP3 & mechanism change\\
\hline
- & - &PMA & PKC & activity + fat-hand \\
- & - &$\beta$2CAMP & PKA & activity + fat-hand \\
\hline
\end{tabular}}
\bigskip
\centerline{\phantom{no-{}}ICAM: \quad \begin{tabular}{|lll|ll|}
\hline
\multicolumn{3}{|c|}{Reagents} & \multicolumn{2}{c|}{Intervention} \\
\hline
$\alpha$-CD3, $\alpha$-CD28 & ICAM-2 & Additional & Target & Type \\
\hline
+ & + &- & - & (observational) \\
+ & + &AKT inhibitor & AKT & activity \\
+ & + &G0076 & PKC & activity \\
+ & + &Psitectorigenin & PIP2 & abundance \\
+ & + &U0126 & MEK & activity \\
+ & + &LY294002 & PIP2/PIP3 & mechanism change\\
\hline
- & - &PMA & PKC & activity + fat-hand \\
- & - &$\beta$2CAMP & PKA & activity + fat-hand \\
\hline
\end{tabular}}
\end{table}

In this paper, we ignore the fact that in the last two interventional datasets in each batch (PMA and $\beta$2CAMP) there is also a global intervention. Ignoring the global intervention allows us to compute the weighted ancestral relations, since we consider any variable that changes its distribution with respect to the observational condition to be an effect of the main target of the intervention (PKC for PMA and PKA for $\beta$2CAMP). This is in line with previous work \cite{SPP05,MooijHeskes_UAI_13}. 
Also, we consider only PIP3 as the main target of the LY294002 intervention, based on the consensus network \cite{SPP05}, even though in \cite{MooijHeskes_UAI_13} both PIP2 and PIP3 are considered to be targets of this intervention. In future work, we plan to extend ACI in order to address the task of learning the intervention targets from data, as done by \cite{EatonMurphy07} for a score-based approach.

In the main paper we provide some results for the most commonly used no-ICAM batch of experimental conditions. Below we report additional results on the same batch. Moreover, we provide results for causal discovery on the ICAM batch, which are quite consistent with the no-ICAM batch. Finally, we compare with other methods that were applied to this dataset, especially with a score-based approach (\cite{MooijHeskes_UAI_13}) that shows surprisingly similar results to ACI, although it uses a very different method.

\begin{figure*}[t!]
\centering     %
\subfigure[Independences of order 0]{\label{fig:noicam_indep}
\includegraphics[width=0.32\textwidth]{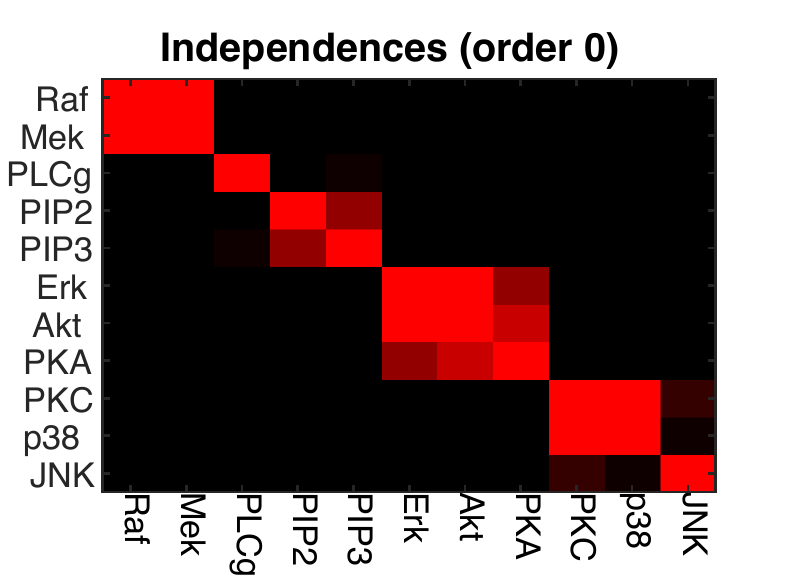}}
\subfigure[Weighted ancestral relations]{\label{fig:noicam_causes}\includegraphics[width=0.32\textwidth]{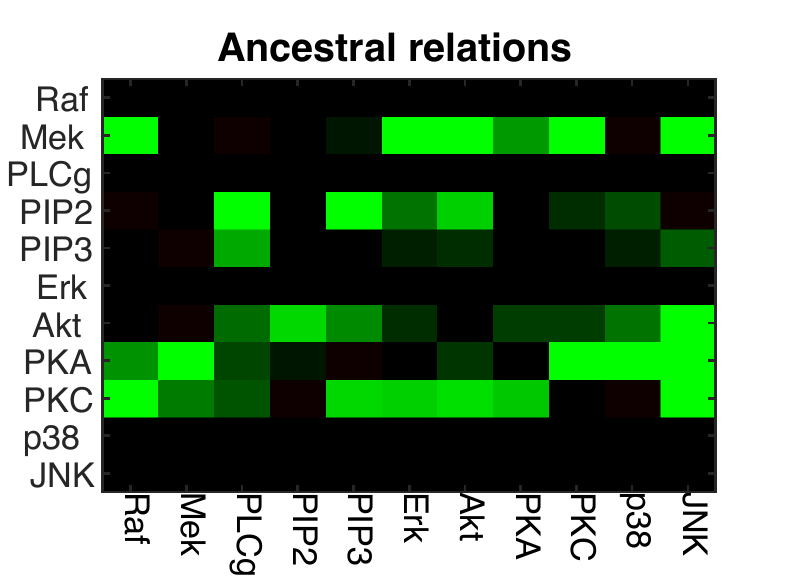}}
\subfigure{\includegraphics[width=0.07\textwidth]{legend.pdf}}
\\
\subfigure[ACI (input: independences order $\leq 1$)]{\label{fig:noicam_aci_indep}
\includegraphics[width=0.3\textwidth]{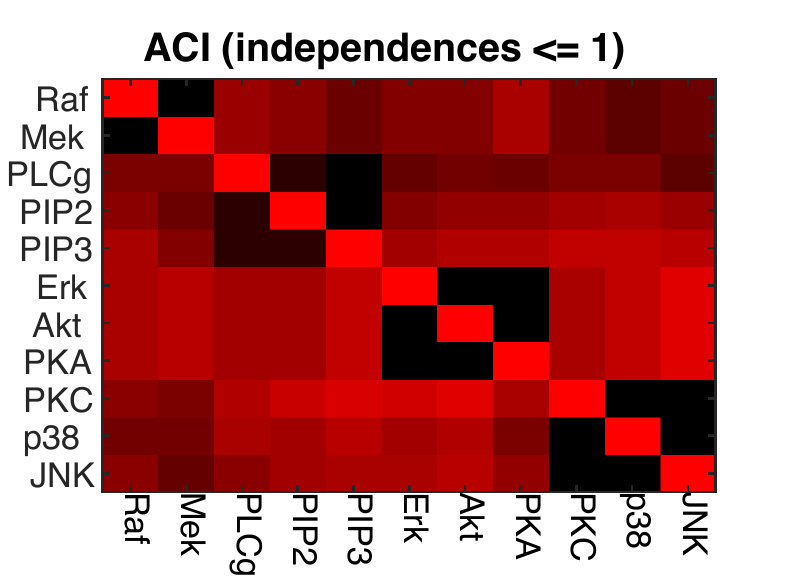}}\quad
\subfigure[ACI (input: weighted ancestral relations)]{\label{fig:noicam_aci_causes}\includegraphics[width=0.3\textwidth]{sachs_aci_anc.pdf}}\quad
\subfigure[ACI (input: independences order $\leq 1$, weighted ancestral relations)]{\label{fig:noicam_aci_both}
\includegraphics[width=0.3\textwidth]{sachs_aci_both.pdf}}
\\
\subfigure[Bootstrapped (100) anytime FCI (input: independences order $\leq 1$)]{\label{fig:noicam_bfci}
\includegraphics[width=0.3\textwidth]{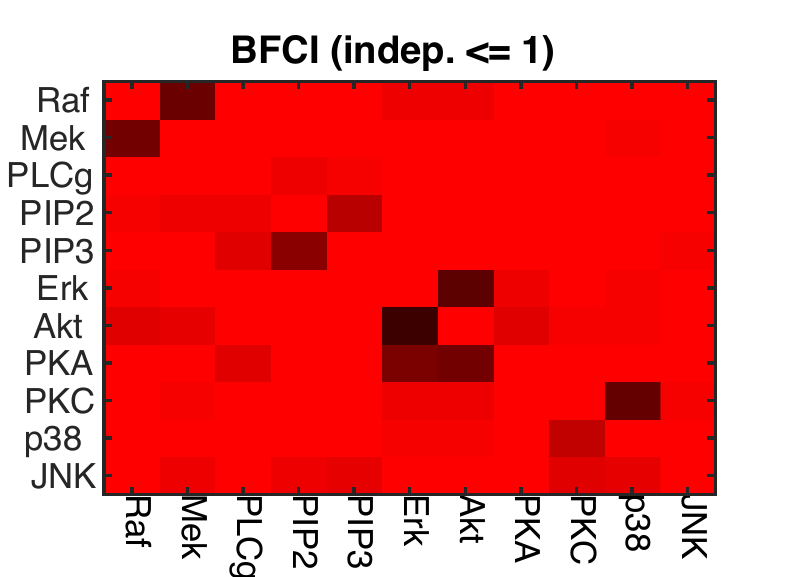}}\quad
\subfigure[Bootstrapped (100) anytime CFCI (input: independences order $\leq 1$)]{\label{fig:noicam_bcfci}\includegraphics[width=0.3\textwidth]{BCFCI_noicam.pdf}}
\caption{\label{fig:noicam}Results on flow cytometry dataset, no-ICAM batch. 
The top row represents some of the possible inputs: weighted independences of order 0 from the observational dataset (the inputs include also order 1 test results, but these are not visualized here) and weighted ancestral relations recovered from comparing the interventional datasets with the observational data.
In the bottom two rows each matrix represents the ancestral relations that are estimated using different inputs and different methods (ACI, bootstrapped anytime FCI or CFCI). Each row represents a cause, while the columns are the effects. The colors encodes the confidence levels, green is positive, black is unknown, while red is negative. The intensity of the color represents the degree of confidence.}
\end{figure*}

\subsection{Results on no-ICAM batch}

In Figure \ref{fig:noicam} we provide additional results for the no-ICAM batch. In the first row we show some of the possible inputs: weighted independences (in this case partial correlations) from observational data and weighted ancestral relations from comparing the interventional datasets with the observational data. Specifically, we consider as inputs only independences up to order $1$ (but only independences of order 0 are visualized in the figure). The color encodes the weight of the independence. As an example, the heatmap shows that Raf and Mek are strongly dependent.

For the weighted ancestral relations, in Figure \ref{fig:noicam} we plot a matrix in which each row represents a cause, while the columns are the effects. 
As described in the main paper we use a $t$-test to test for each intervention and for each variable whether its distribution changes with respect to the observational condition. We use the biological knowledge summarised in Table \ref{sachsexperiments} to define the intervention target, which is then considered the putative ``cause''.
Then we use the $p$-values of these tests and a threshold $\alpha = 0.05$ to obtain the weights of the ancestral relations, similarly to what is proposed in the main paper for the frequentist weights for the independence tests:
\begin{equation*}
w = |\log p - \log \alpha|.
\end{equation*}
For example, if adding U0126 (which is known to be a MEK inhibitor) changes the distribution of RAF with $p=0.01$ with respect to the observational baseline, we get a 
weighted ancestral relation (MEK$\causes$RAF, 1.609). 

\begin{figure*}[t!]
\centering     %
\subfigure[Independences of order 0]{\label{fig:icam_indep}
\includegraphics[width=0.32\textwidth]{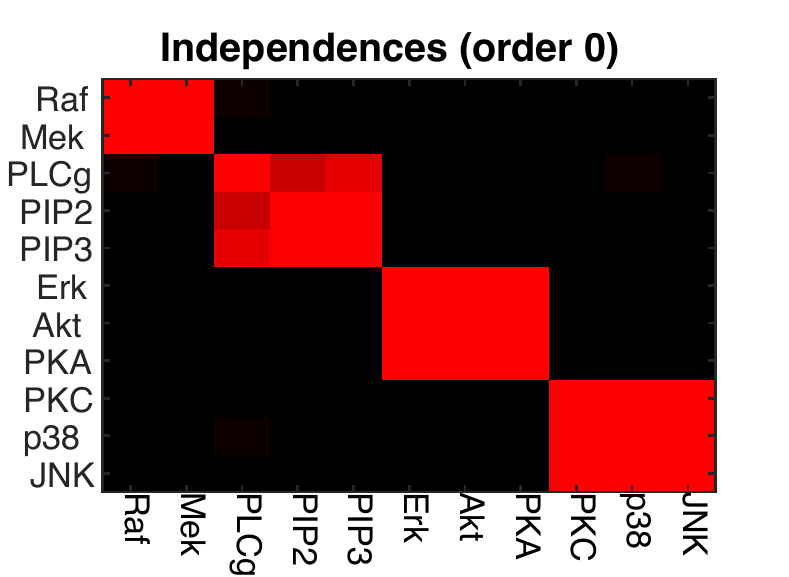}}
\subfigure[Weighted ancestral relations]{\label{fig:icam_causes}\includegraphics[width=0.32\textwidth]{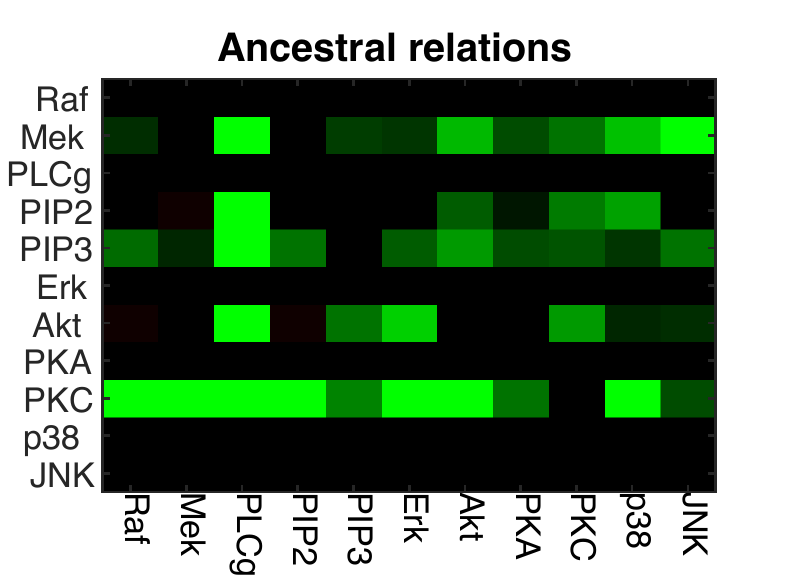}}
\subfigure{\includegraphics[width=0.07\textwidth]{legend.pdf}}
\\
\subfigure[ACI (input: independences order $\leq 1$)]{\label{fig:icam_aci_indep}
\includegraphics[width=0.3\textwidth]{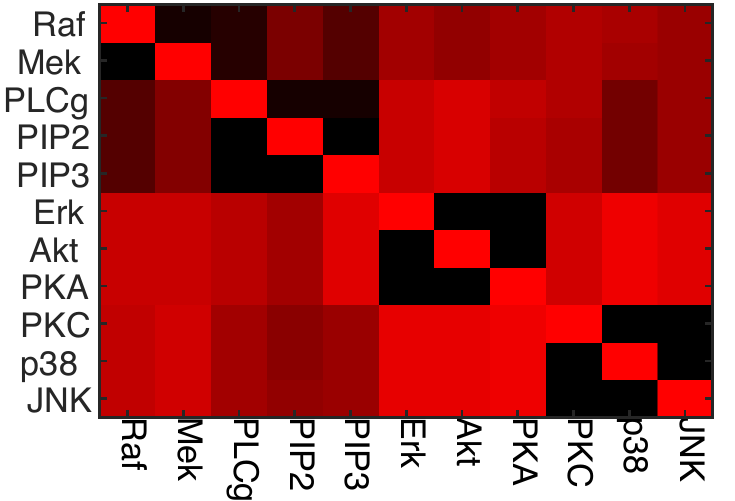}}\quad
\subfigure[ACI (input: weighted ancestral relations)]{\label{fig:icam_aci_causes}\includegraphics[width=0.3\textwidth]{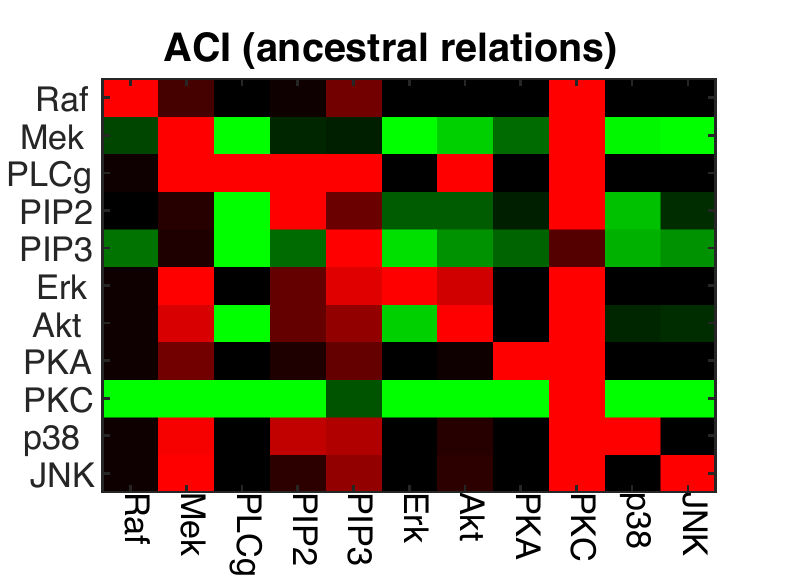}}\quad
\subfigure[ACI (input: independences order $\leq 1$, weighted ancestral relations)]{\label{fig:icam_aci_both}
\includegraphics[width=0.3\textwidth]{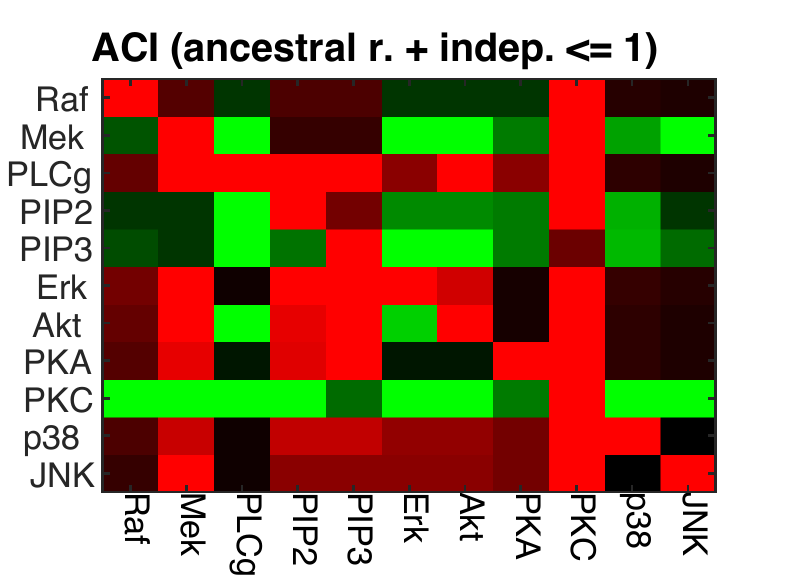}}
\\
\subfigure[Bootstrapped (100) anytime FCI(input: independences order $\leq 1$)]{\label{fig:icam_bfci}
\includegraphics[width=0.3\textwidth]{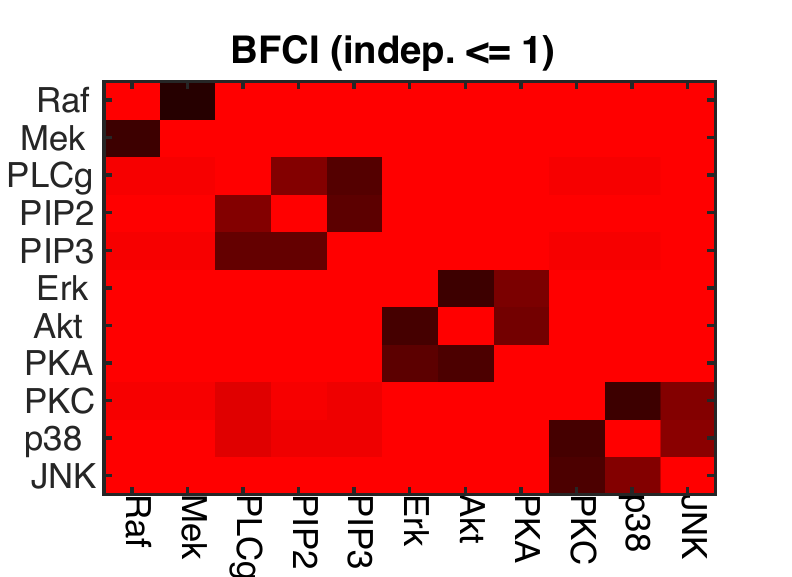}}\quad
\subfigure[Bootstrapped (100) anytime CFCI (input: independences order $\leq 1$)]{\label{fig:icam_bcfci}\includegraphics[width=0.3\textwidth]{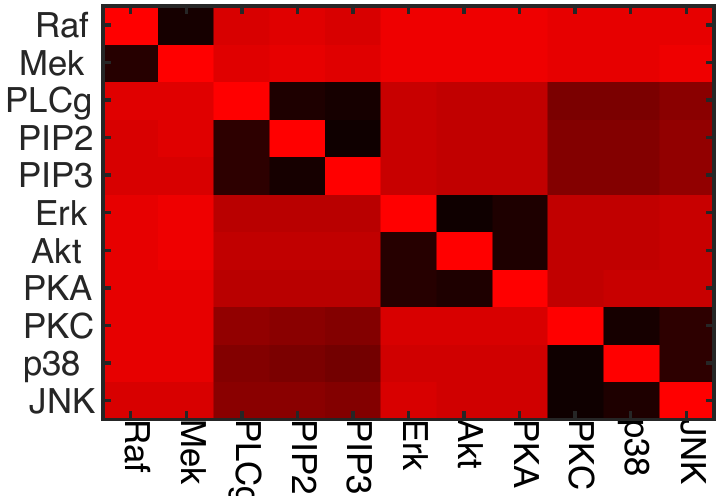}}
  \caption{\label{fig:icam}Results on flow cytometry dataset, ICAM batch. Same comparison as in Figure \ref{fig:noicam}, but for the ICAM batch.}
\end{figure*}

\subsection{ICAM batch}
In Figure \ref{fig:icam} we show the results for the ICAM setting. These results are very similar to the results for the no-ICAM batch (see also Figure~\ref{fig:noicamicam}), showing that the predicted ancestral relations are robust. In particular it is clear that also for the ICAM batch, weighted ancestral relations are a very strong signal, and that methods that can exploit them (e.g., ACI) have a distinct advantage over methods that cannot (e.g., FCI and CFCI).

In general, in both settings there appear to be various faithfulness violations.
For example, it is well-known that MEK causes ERK, yet in the
observational data these two variables are independent. Nevertheless, we can
see in the data that an intervention on MEK leads to a change of ERK, as
expected. It is interesting to note that our approach can correctly recover
from this faithfulness violation because it takes into
account the weight of the input statements (note that the weight of the independence
is smaller than that of the ancestral relation, which corresponds with a quite 
significant change in distribution). 
In contrast, methods that start by reconstructing the skeleton (like (C)FCI or LoCI \cite{loci}) would decide that MEK and ERK are nonadjacent, unable to recover from that 
erroneous decision. This illustrates one of the advantages of our approach.

\begin{figure}[t]
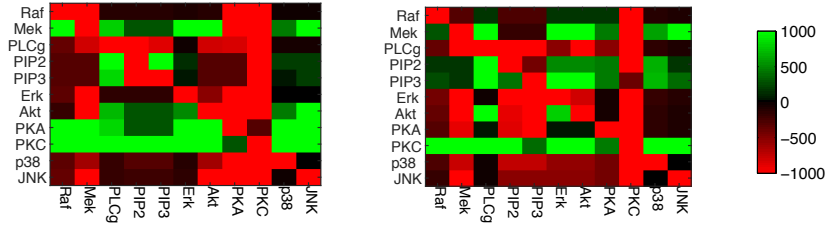

\centering
\includegraphics[width=0.3\textwidth]{sachs_aci_both.pdf}
\qquad
\includegraphics[width=0.3\textwidth]{icam_aci.pdf}
\qquad
\includegraphics[width=0.07\textwidth]{legend.pdf}
  \caption{ACI results (input: independences of order $\leq 1$ and weighted ancestral relations) on no-ICAM (left) and ICAM (right) batches. These heatmaps are identical to the ones in Figures~\ref{fig:noicam} and \ref{fig:icam}, but are reproduced here next to each other for easy comparison.\label{fig:noicamicam}}
\end{figure}

\subsection{Comparison with other approaches}
We also compare our results with other, mostly score-based approaches. 
Amongst other results, \cite{MooijHeskes_UAI_13} report the top 17 direct causal relations on the no-ICAM batch that were inferred by their score-based method when assuming acyclicity. In order to compare fairly with the ancestral relations found by ACI, we first perform a transitive closure of these direct causal relations, which results in 21 ancestral relations. We then take the top 21 predicted ancestral relations from ACI (for the same no-ICAM batch), and compare the two in Figure \ref{fig:comparisonMooij13}. The black edges, the majority, represent the ancestral relations found by both methods.  The blue edges are found only by ACI, while the grey edges are found only by \cite{MooijHeskes_UAI_13}. Interestingly, the results are quite similar, despite the very different approaches. In particular, ACI allows for confounders and is constraint-based, while the method in \cite{MooijHeskes_UAI_13} assumes causal sufficiency (i.e., no confounders) and is score-based.

\begin{figure*}[t!]
\centering 
\includegraphics[width=0.8\textwidth]{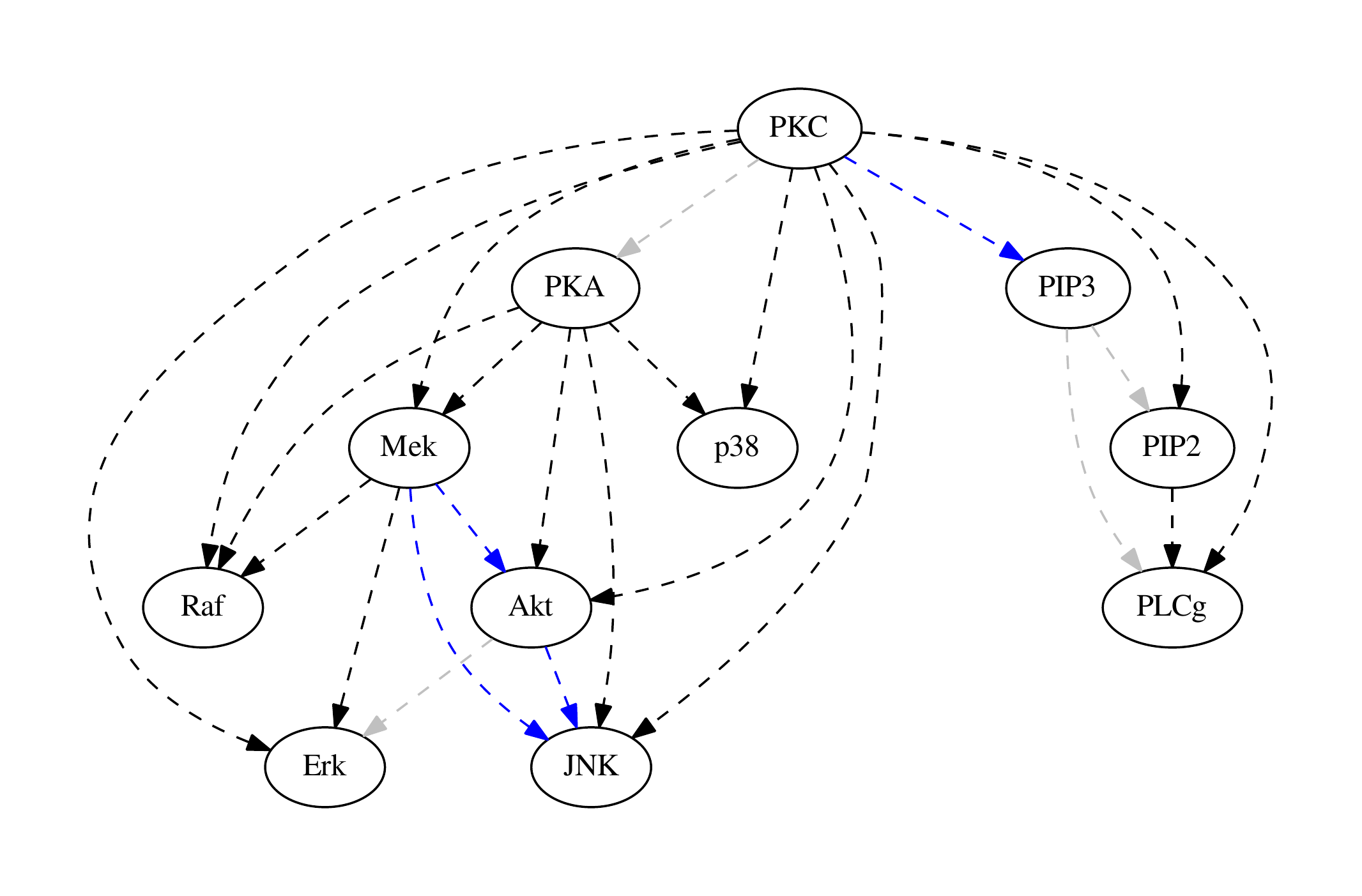}
  \caption{Comparison of ancestral relations predicted by ACI and the score-based method from \cite{MooijHeskes_UAI_13}, both using the no-ICAM batch. Depicted are the top 21 ancestral relations obtained by ACI and the transitive closure of the top 17 direct causal relations reported in \cite{MooijHeskes_UAI_13}, which results in 21 ancestral relations. Black edges are ancestral relations found by both methods, blue edges were identified only by ACI, while grey edges are present only in the transitive closure of the result from \cite{MooijHeskes_UAI_13}.\label{fig:comparisonMooij13}}
\end{figure*}

Table \ref{tab:comparisonSachs} summarizes most of the existing work on this flow cytometry dataset. It was originally part of the S1 material of \cite{ICP_PNAS}. We have updated it here by adding also the results for ACI and the transitive closure of \cite{MooijHeskes_UAI_13}.

\begin{table*}
\caption{Updated Table S1 from \cite{ICP_PNAS}: causal relationships between the biochemical agents in the flow cytometry data of \cite{SPP05}, according to different causal discovery methods. The consensus network according to \cite{SPP05} is denoted here by ``\cite{SPP05}a'' and their reconstructed network by ``\cite{SPP05}b''. For \cite{MooijHeskes_UAI_13} we provide two versions: ``\cite{MooijHeskes_UAI_13}a'' for the top 17 edges in the acyclic case, as reported in the original paper, and ``\cite{MooijHeskes_UAI_13}b'' for its transitive closure, which consists of 21 edges.
To provide a fair comparison, we also pick the top 21 ancestral predictions from ACI.\label{tab:comparisonSachs}}
{\small
\begin{tabular}{l|ccccc|ccc}
\hline
  & \multicolumn{5}{|c|}{Direct causal predictions} & \multicolumn{3}{|c}{Ancestral predictions} \\%
Edge & \cite{SPP05}a & \cite{SPP05}b & \cite{MooijHeskes_UAI_13}a  & \cite{EatonMurphy07} & ICP \cite{ICP} & hiddenICP \cite{ICP} & \cite{MooijHeskes_UAI_13}b & ACI (top 21)\\
\hline
RAF$\to$MEK   & $\ck$ & $\ck$ &         	&
	        &          & $\ck$ &          & 
\\
MEK$\to$RAF   &				&           & $\ck$ & 
$\ck$ 	&          & $\ck$ &  $\ck$ &
$\ck$ \\
MEK$\to$ERK   & $\ck$ & $\ck$ & $\ck$ & 
            &         &            & $\ck$ &
$\ck$\\
MEK$\to$AKT   & 		   &  		   &  		   & 
            &         &            &  		    &
$\ck$\\
MEK$\to$JNK   & 		   &  		   &  		   & 
            &         &            &  		    &
$\ck$\\
PLCg$\to$PIP2 & $\ck$ & $\ck$ &          &	
$\ck$ 	& $\ck$ & $\ck$ & 		     & 
\\
PLCg$\to$PIP3 &           & $\ck$ &          &
$\ck$ 	 &         &           &			 &  
\\
PLCg$\to$PKC  & $\ck$ &           &		  &
$\ck$ 	 &         &            & 			 &
\\
PIP2$\to$PLCg &            &          & $\ck$&
			 &$\ck$&            & $\ck$ &
$\ck$\\
PIP2$\to$PIP3  &            &           &         &
$\ck$ 	 &         &            &			 &
\\
PIP2$\to$PKC   & $\ck$  &          &          &
			 &         &             &			  &
\\
PIP3$\to$PLCg & $\ck$  &           &          & 
			 &         &            &$\ck$ &
\\
PIP3$\to$PIP2  & $\ck$  & $\ck$ & $\ck$ &  
            &$\ck$& $\ck$   & $\ck$ &
\\
PIP3$\to$AKT  & $\ck$   &          &           & 
		     &         &             &          &
\\
AKT$\to$ERK   &             &         & $\ck$ &
            &$\ck$& $\ck$   &	$\ck$ &
\\
AKT$\to$JNK   & 		   &  		   &  		   & 
            &         &            &  		    &
$\ck$\\
ERK$\to$AKT   &             & $\ck$ &         & 
$\ck$ 	 & $\ck$ & $\ck$ &				&
\\
ERK$\to$PKA   &              &          &          & 
$\ck$   &    		 &            	&            &
\\
PKA$\to$RAF   & $\ck$    & $\ck$ &          &
            &          &              &   $\ck$ &    
$\ck$ \\
PKA$\to$MEK   & $\ck$    & $\ck$ & $\ck$ & 
$\ck$  &           & $\ck$    &  $\ck$ &
$\ck$\\
PKA$\to$ERK   & $\ck$     & $\ck$ &          &
          & $\ck$ &               &  $\ck$ &
$\ck$\\
PKA$\to$AKT   & $\ck$    & $\ck$ & $\ck$ &
$\ck$ &            & $\ck$    &  $\ck$ &
$\ck$\\
PKA$\to$PKC   &              &           &           & 
$\ck$ &            &             &			&
\\
PKA$\to$P38   & $\ck$ 	  & $\ck$  & $\ck$ &
          &            &             & $\ck$ &
$\ck$\\
PKA$\to$JNK   & $\ck$ 	  & $\ck$  & $\ck$ &
$\ck$ &            &             & $\ck$ &
$\ck$ \\
PKC$\to$RAF    & $\ck$   & $\ck$ & $\ck$  & 
         &              &            &$\ck$ & 
$\ck$\\
PKC$\to$MEK   & $\ck$ & $\ck$ & $\ck$ &
$\ck$ &          &           &  $\ck$ & 
$\ck$\\
PKC$\to$PLCg  &              &              & $\ck$  &     
         &              &              &$\ck$ & 
$\ck$\\
PKC$\to$PIP2  &              &              & $\ck$  &
         &              &              &  $\ck$ &          
$\ck$\\
PKC$\to$PIP3   &              &           &           & 
 		&            &             &			&
$\ck$\\
PKC$\to$ERK   &              &              &   	      & 
       &              &               & $\ck$ &
$\ck$\\
PKC$\to$AKT   &              &              & $\ck$ & 
       &              &               & $\ck$ &
$\ck$\\
PKC$\to$PKA   &              & $\ck$ & $\ck$ &  
            &              &              & $\ck$ &
\\
PKC$\to$P38   & $\ck$ & $\ck$ & $\ck$ & 
$\ck$ &              & $\ck$ &  $\ck$ & 
$\ck$ \\
PKC$\to$JNK   & $\ck$ & $\ck$ & $\ck$ &
$\ck$ & $\ck$ & $\ck$ &  $\ck$ &
$\ck$ \\
P38$\to$JNK   &              &              &              &
             &              & $\ck$ &      &
\\
P38$\to$PKC   &              &              &              & 
            &              & $\ck$ &          &
\\
JNK$\to$PKC   &              &              &              & 
           &              & $\ck$ &              &
\\
JNK$\to$P38   &              &              &              & 
$\ck$ &              & $\ck$ &             &
 \\
\hline
\end{tabular}
}
\end{table*}

\section{Mapping ancestral structures to direct causal relations}
An ancestral structure can be seen as the transitive closure of the directed edges of an acyclic directed mixed graph (ADMG). 
There are several strategies to reconstruct ``direct'' causal relations from an ancestral structure, in particular in combination with our scoring method. Here we sketch a possible strategy, but we leave a more in-depth investigation to future work.

\begin{figure*}[t!]
\centering     %
\subfigure[PR direct]{
\includegraphics[width=0.47\textwidth]{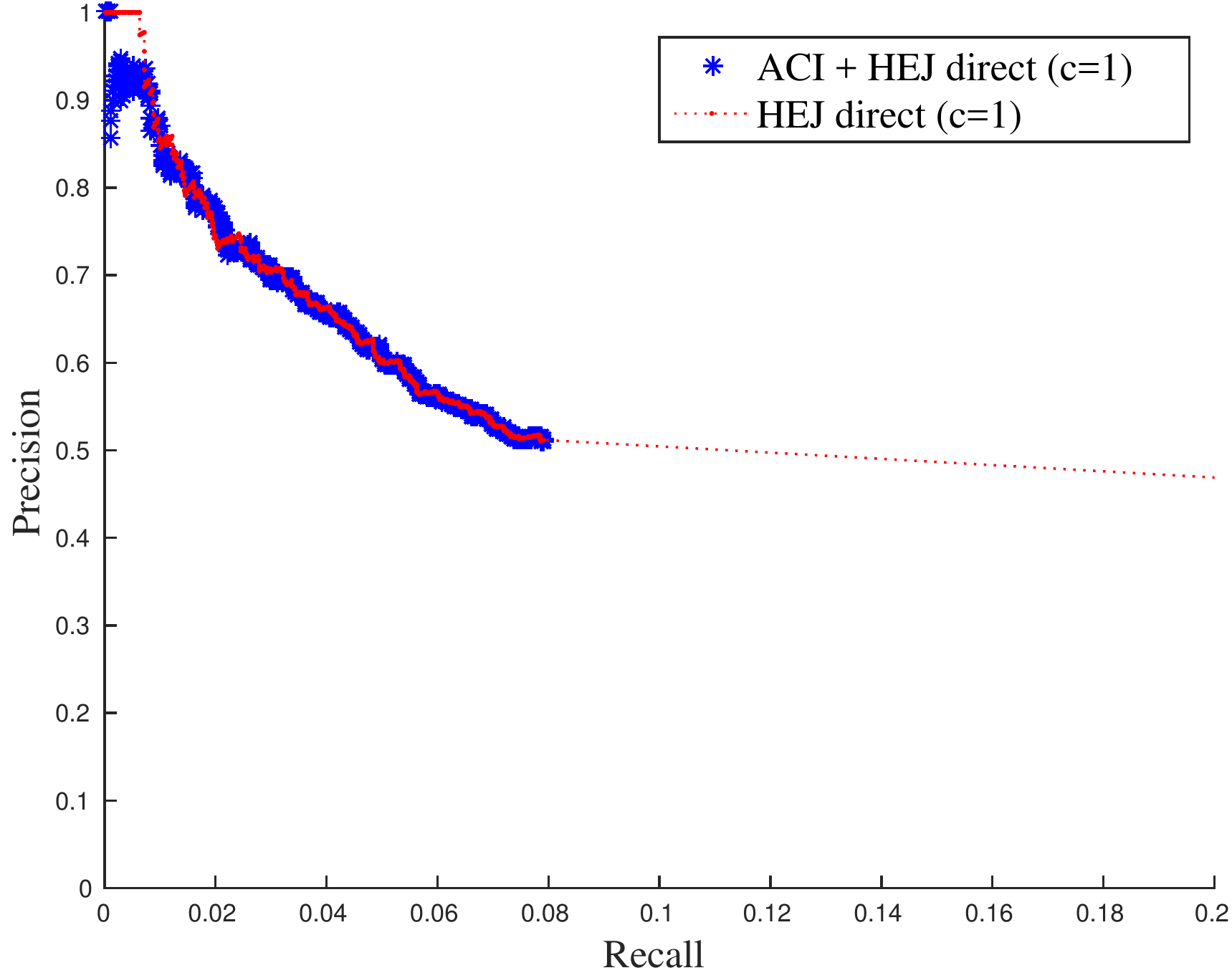}}
\subfigure[PR direct (zoom)]{
\includegraphics[width=0.47\textwidth]{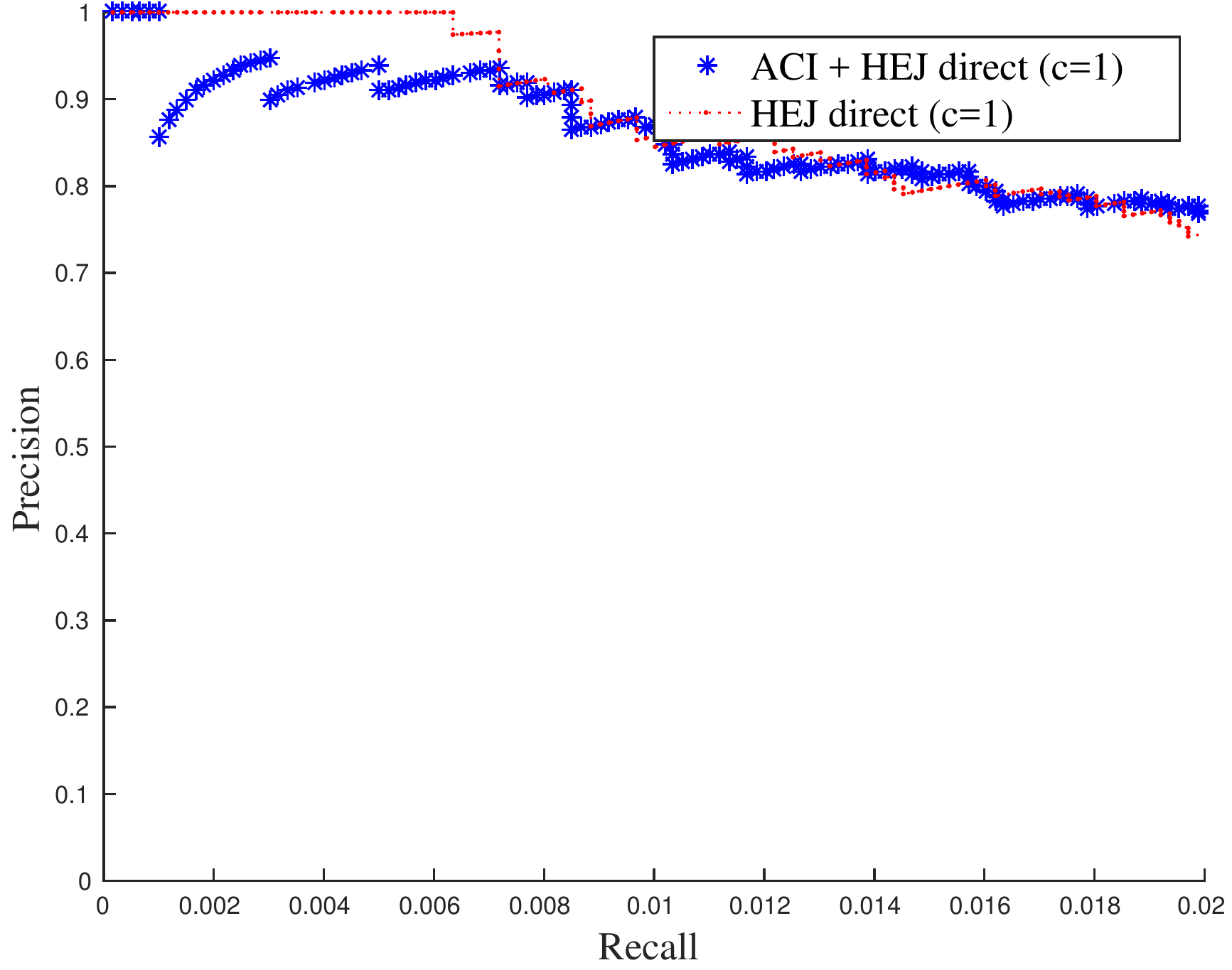}}
\\
\subfigure[PR direct acausal]{
\includegraphics[width=0.47\textwidth]{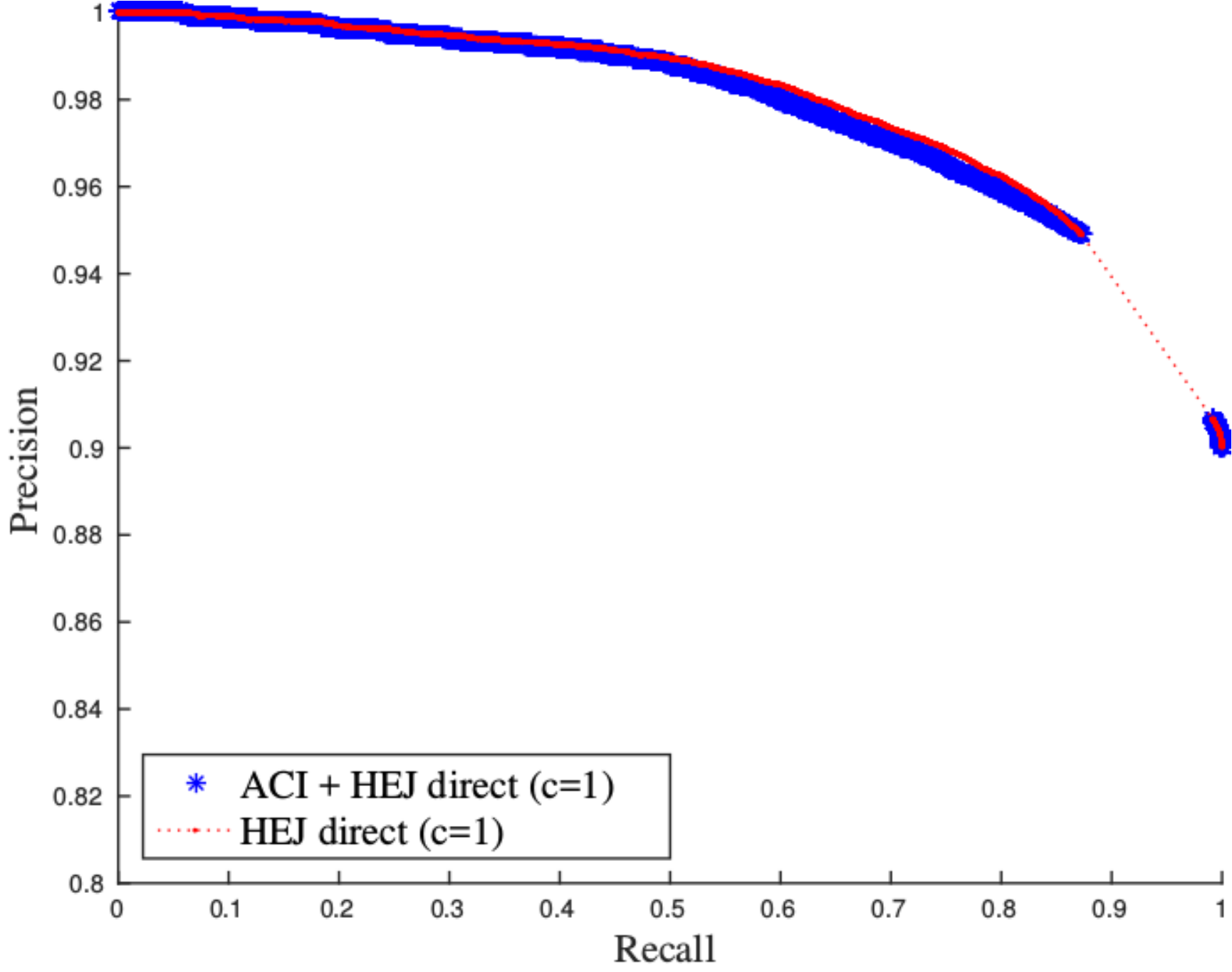}}
\caption{\label{fig:direct_acc}Synthetic data: accuracy for the two prediction tasks (direct causal and noncausal relations) for $n=6$ variables using the frequentist test with $\alpha=0.05$ for 2000 simulated models.}
\end{figure*}

\begin{table}
    \centering%
    \caption{Average execution times for recovering causal relations with different strategies for 2000 models for $n=6$ variables using the frequentist test with $\alpha=0.05$.\\\label{extime2}}
    \begin{tabular}{ |l|l|l|l|| l ||l|}%
    \hline%
      \multicolumn{6}{|c|}{\textbf{Average execution time (s)}} \\%
            \hline \hline
      \multicolumn{2}{|c|}{Setting}  &
      \multicolumn{2}{|c||}{\textbf{Direct causal relations}} &
      \multicolumn{1}{|c||}{\textbf{Only second step}} &
      \multicolumn{1}{|c|}{\textbf{Ancestral relations}} \\%
      \hline
    $n$ &$c$& \name\ with restricted HEJ & direct HEJ & restricted HEJ & ancestral HEJ\\ \hline
    6 & 1 & 9.77 & 15.03 & 7.62 & 12.09\\ \hline
    6 & 4 & 16.96 & 314.29 & 14.43 & 432.67\\ \hline
    7 & 1 & 36.13 & 356.49 & 30.68 &	715.74\\ \hline
    8 & 1 & 98.92 & $\geq 2500$ & 81.73 &  $\geq 2500 $\\ \hline
    9 & 1 & 361.91 &  $\geq 2500$ & 240.47 &  $\geq 2500$ \\ \hline
    \end{tabular}
\end{table}

A possible strategy is to first recover the ancestral structure from ACI with our scoring method and then use it as ``oracle'' input constraints for the HEJ \cite{antti} algorithm. Specifically, for each weighted output $(X \causes Y, w)$ obtained by ACI, we add $(X \causes Y, \infty)$ to the input list $I$, and similarly for each $X \notcauses Y$. Then we can use our scoring algorithm with HEJ to score direct causal relations (e.g., $f = X \to Y$) and direct acausal relations (e.g., $f= X \not \to Y$):
\begin{equation}
  C(f) = \min_{W \in \mathcal{W}} \loss(W;I \cup \{(\lnot f,\infty)\}) - \min_{W \in \mathcal{W}} \loss(W;I \cup \{(f,\infty)\}).
\end{equation}
In the standard HEJ algorithm, $\mathcal{W}$ are all possible ADMGs, but with our additional constraints we can reduce the search space to only the ones that fit the specific ancestral structure, which is on average and asymptotically a reduction of $2^{n^2/4 + o(n^2)}$ for $n$ variables. We will refer to this two-step approach as \emph{ACI with restricted HEJ (ACI + HEJ)}. A side effect of assigning infinite scores to the original ancestral predictions instead of the originally estimated scores is that some of the estimated direct causal predictions scores will also be infinite, flattening their ranking. For this preliminary evaluation, we fix this issue by reusing the original ancestral scores also for the infinite direct predictions scores. Another option may be to use the ACI scores for (a)causal relations as soft constraints for HEJ, although at the time of writing it is still unclear whether this would lead to the same speedup as the previously mentioned version.

We compared accuracy and execution times of standard HEJ (without the additional constraints derived from ACI) with ACI with restricted HEJ on simulated data. Figure~\ref{fig:direct_acc} shows PR curves for predicting the presence and absence of direct causal relations for both methods. In Table~\ref{extime2} we list the execution times for recovering direct causal relations. Additionally, we list the execution times of only the second step of our approach, the \emph{restricted HEJ}, to highlight the improvement in execution time resulting from the restrictions. In this preliminary investigation with simulated data, ACI with restricted HEJ is much faster than standard HEJ (without the additional constraints derived from ACI) for predicting direct causal relations, but only sacrifices a little accuracy (as can be seen in Figure \ref{fig:direct_acc}). In the last column of Table~\ref{extime2}, we show the execution times of standard HEJ when used to score ancestral relations. Interestingly, predicting direct causal relations is faster than predicting ancestral relations with HEJ. Still, for 8 variables the algorithm takes more than 2,500 seconds for all but 6 models of the first 40 simulated models.

Another possible strategy first reconstructs the (possibly incomplete) PAG \cite{Spirtes2000} from ancestral relations and conditional (in)dependences using a procedure similar to LoCI \cite{loci}, and then recovering direct causal relations. There are some subtleties in the conversion from (possibly incomplete) PAGs to direct causal relations, so we leave this and other PAG based strategies, as well as a better analysis of conversion of ancestral relations to direct causal relations as future work.

\section{Complete ACI encoding in ASP}
Answer Set Programming (ASP) is a widely used declarative programming language based on the stable model semantics of logical programming. A thorough introduction to ASP can be found in \cite{DBLP:conf/aaai/Lifschitz08,DBLP:reference/fai/Gelfond08}. 
The ASP syntax resembles Prolog, but the computational model is based on the principles that have led to faster solvers for propositional logic \cite{DBLP:conf/aaai/Lifschitz08}. 
ASP has been applied to several NP-hard problems, including learning Bayesian networks and ADMGs \cite{antti}. Search problems are reduced to computing the stable models (also called answer sets), which can be optionally scored.

For ACI we use the state-of-the-art ASP solver \texttt{clingo 4} \cite{clingo}. We provide the complete ACI encoding in ASP using the clingo syntax in Table \ref{tab:encoding}.
We encode sets via their natural correspondence with binary numbers and use boolean formulas in ASP to encode set-theoretic operations. Since ASP does not support real numbers, we scale all weights by a factor of 1000 and round to the nearest integer.

\begin{table*}[p]
\caption{Complete ACI encoding in Answer Set Programming, written in the syntax for clingo 4.\label{tab:encoding}}
\begin{lstlisting}
%%%%%%%%%%%%%%%%%%%%%%%%%%%%%%%%%%%%%%%%%%%%%%%%
%%%%%  Ancestral Causal Inference (ACI)    %%%%%
%%%%%%%%%%%%%%%%%%%%%%%%%%%%%%%%%%%%%%%%%%%%%%%%

%%%%% Preliminaries:
%%% Define ancestral structures:
{ causes(X,Y) } :- node(X), node(Y), X!=Y.
:- causes(X,Y), causes(Y,X), node(X), node(Y), X < Y.
:- not causes(X,Z), causes(X,Y), causes(Y,Z), node(X), node(Y), node(Z).

%%% Define the extension of causes to sets.
% existsCauses(Z,W) means there exists I \in W that is caused by Z.
1{causes(Z, I): ismember(W,I)} :- existsCauses(Z,W), node(Z), set(W), not ismember(W,Z).
existsCauses(Z,W) :- causes(Z, I), ismember(W,I), node(I), node(Z), set(W), not ismember(W,Z), Z!=I.

%%% Generate in/dependences in each model based on the input in/dependences.
1{ dep(X,Y,Z);indep(X,Y,Z) }1 :- input_indep(X,Y,Z,_).
1{ dep(X,Y,Z);indep(X,Y,Z) }1 :- input_dep(X,Y,Z,_).

%%% To simplify the rules, add symmetry of in/dependences.
dep(X,Y,Z) :- dep(Y,X,Z), node(X), node(Y), set(Z), X!=Y, not ismember(Z,X), not ismember(Z,Y).
indep(X,Y,Z) :- indep(Y,X,Z), node(X), node(Y), set(Z), X!=Y, not ismember(Z,X), not ismember(Z,Y).

%%%%% Rules from LoCI:
%%% Minimal independence rule (4) : X || Y | W u [Z] => Z -/-> X, Z -/-> Y, Z -/-> W
:- not causes(Z,X), not causes(Z,Y), not existsCauses(Z,W), dep(X,Y,W), indep(X,Y,U), 
U==W+2**(Z-1), set(W), node(Z), not ismember(W, Z), Y != Z, X != Z.

%%% Minimal dependence rule (5): X |/| Y | W u [Z] => Z --> X or Z-->Y or Z-->W 
:- causes(Z,X), indep(X,Y,W), dep(X,Y,U), U==W+2**(Z-1), set(W), set(U), node(X),
 node(Y), node(Z), not ismember(W, Z), not ismember(W, X), not ismember(W,Y), 
 X != Y, Y != Z, X != Z.
% Note: the version with causes(Z,Y) is implied by the symmetry of in/dependences.
:- existsCauses(Z,W), indep(X,Y,W), dep(X,Y,U), U==W+2**(Z-1), set(W), set(U), node(X),
 node(Y), node(Z), not ismember(W, Z), not ismember(W, X), not ismember(W,Y), 
 X != Y, Y != Z, X != Z.

%%%%% ACI rules:
%%% Rule 1: X || Y | U and X -/-> U => X -/->Y
:- causes(X,Y), indep(X,Y,U), not existsCauses(X,U), node(X), node(Y), set(U), X != Y, 
not ismember(U,X), not ismember(U,Y).

%%% Rule 2: X || Y | W u [Z] => X |/| Z | W 
dep(X,Z,W) :- indep(X,Y,W), dep(X,Y,U), U==W+2**(Z-1), set(W), set(U), 
node(X), node(Y), node(Z), X != Y, Y != Z, X != Z, not ismember(W,X), not ismember(W,Y).

%%% Rule 3: X |/| Y | W u [Z] => X |/| Z | W
dep(X,Z,W) :- dep(X,Y,W), indep(X,Y,U), U==W+2**(Z-1), set(W), set(U), 
node(X), node(Y), node(Z), X != Y, Y != Z, X != Z, not ismember(W,X), not ismember(W,Y).

%%% Rule 4: X || Y | W u [Z] and X || Z | W u U => X || Y | W u U
indep(X,Y,A) :- dep(X,Y,W), indep(X,Y,U), U==W+2**(Z-1), indep(X,Z,A), A==W+2**(B-1),
 set(W), set(U), not ismember(W,X), not ismember(W,Y), node(X), node(Y), node(Z), 
 set(A), node(B), X!=B, Y!=B, Z!=B, X != Y, Y != Z, X != Z.

%%% Rule 5: Z |/| X | W and Z |/| Y | W and X || Y | W => X |/| Z | W u Z
dep(X,Y,U) :- dep(Z,X,W), dep(Z,Y,W), indep(X,Y,W), node(X), node(Y), U==W+2**(Z-1),
 set(W), set(U), X != Y, Y != Z, X != Z, not ismember(W,X), not ismember(W,Y).

%%%%% Loss function and optimization.
%%% Define the loss function as the incongruence between the input in/dependences 
%%% and the in/dependences of the model.
fail(X,Y,Z,W) :- dep(X,Y,Z), input_indep(X,Y,Z,W).
fail(X,Y,Z,W) :- indep(X,Y,Z), input_dep(X,Y,Z,W).

%%% Include the weighted ancestral relations in the loss function.
fail(X,Y,-1,W) :- causes(X,Y), wnotcauses(X,Y,W), node(X), node(Y), X != Y.
fail(X,Y,-1,W) :- not causes(X,Y), wcauses(X,Y,W), node(X), node(Y), X != Y.

%%% Optimization part: minimize the sum of W of all fail predicates that are true.
#minimize{W,X,Y,C:fail(X,Y,C,W) }.
\end{lstlisting}
\end{table*} 

\section{Open source code repository}
We provide an open-source version of our algorithms and the evaluation framework, which can be easily extended, at \url{http://github.com/caus-am/aci}.

\clearpage
\bibliographystyle{abbrv}
\bibliography{biblio}

\end{document}